\documentclass{article}

     \PassOptionsToPackage{numbers, compress}{natbib}
 \usepackage[preprint]{neurips_2025}

\usepackage[utf8]{inputenc} 
\usepackage[T1]{fontenc}    
\usepackage{hyperref}       
\usepackage{url}             
\usepackage{booktabs}       
\usepackage{amsfonts}       
\usepackage{nicefrac}       
\usepackage{microtype}      
\usepackage{xcolor}         

\usepackage{comment}
\usepackage{graphicx}
\usepackage{amsmath,amssymb, mathtools}
\usepackage{amsthm}

\newtheorem{definition}{Definition}

\newtheorem{lemma}{Lemma}
\newtheorem{proposition}{Proposition}

\usepackage{enumitem}

\newtheorem{innercustomlem}{Lemma}
\newenvironment{customlem}[1]
  {\renewcommand\theinnercustomlem{#1}\innercustomlem}
  {\endinnercustomlem}

\newtheorem{innercustomprop}{Proposition}
\newenvironment{customprop}[1]
  {\renewcommand\theinnercustomprop{#1}\innercustomprop}
  {\endinnercustomprop}

\DeclareMathOperator{\Prob}{\mathbb{P}}
\DeclareMathOperator{\E}{\mathbb{E}}  
\DeclareMathOperator{\IndicatorFunc}{\mathbf{1}}  

\usepackage[skins,theorems]{tcolorbox}
\usepackage{graphicx}

\definecolor{backg_blue}{RGB}{240, 248, 255}
\definecolor{lightorange}{RGB}{255, 200, 124}

\title{Balance Equation-based Distributionally Robust Offline Imitation Learning}

\author{%
  Rishabh Agrawal \\
  University of Southern California\\
  \texttt{rishabha.edu} \\
   \And
   Yusuf Alvi \\
  University of Southern California\\
  \texttt{yalvi@usc.edu} \\
   \And
   Rahul Jain \\
  University of Southern California\\
  \texttt{rahul.jain@usc.edu} \\
   \And
   Ashutosh Nayyar \\
  University of Southern California\\
  \texttt{ashutosn@usc.edu} \\
}

\begin{document}

\maketitle

\begin{abstract}
  Imitation Learning (IL) has proven highly effective for robotic and control tasks where manually designing reward functions or explicit controllers is infeasible. However, standard IL methods implicitly assume that the environment dynamics remain fixed between training and deployment. In practice, this assumption rarely holds where modeling inaccuracies, real-world parameter variations, and adversarial perturbations can all induce shifts in transition dynamics, leading to severe performance degradation. We address this challenge through Balance Equation-based Distributionally Robust Offline Imitation Learning, a framework that learns robust policies solely from expert demonstrations collected under nominal dynamics, without requiring further environment interaction. We formulate the problem as a distributionally robust optimization over an uncertainty set of transition models, seeking a policy that minimizes the imitation loss under the worst-case transition distribution. Importantly, we show that this robust objective can be reformulated entirely in terms of the nominal data distribution, enabling tractable offline learning. Empirical evaluations on continuous-control benchmarks demonstrate that our approach achieves superior robustness and generalization compared to state-of-the-art offline IL baselines, particularly under perturbed or shifted environments.
\end{abstract}

\section{Introduction}

Imitation Learning (IL) has become a cornerstone of modern AI, enabling agents to acquire complex skills in domains where manually engineering a reward function is difficult or impractical \citep{zitkovich2023rt}. In applications ranging from robotic manipulation to autonomous driving, it is often far simpler to provide expert demonstrations than to design a dense, well-shaped reward signal that captures the desired behavior \citep{reed2022generalist}. IL frameworks leverage these demonstrations to learn control policies directly, bypassing the challenges of reward engineering \citep{pmlr-v267-bai25e}. However, this success is predicated on the fragile assumption that the transition dynamics at test time match those under which the demonstrations were collected \citep{viano2021robust,chae2022robust}. In real-world systems, this rarely holds \citep{peng2018sim}. Physical parameters such as mass, friction, or motor response can drift over time \citep{lyu2024odrl}; unmodeled effects like actuator delay or unexpected contacts may arise \citep{wang2023addressing,jiang2024contact}; and the environment itself can undergo adversarial perturbations \citep{hsu2024reforma}. 
This creates a \emph{robustness gap}, as policies effective in training environments degrade under perturbed dynamics \citep{hoque2024intervengen}.

A natural way to mitigate this robustness gap is to introduce interaction across multiple source environments, the target environment, or both during training. Under multi-source interaction, RIME \citep{chae2022robust} learns a single policy that generalizes over varying dynamics by minimizing a Jensen–Shannon risk across expert policies, while ADAIL \citep{lu2020adail} leverages dynamics embeddings and domain-adversarial objectives to enforce invariance. In source-only interaction, AIRL \citep{fu2018learning} recovers reward functions that disentangle expert intent from environment-specific dynamics. For target-only adaptation, DYNAIL \citep{liu2023dynamics} fine-tunes policies by distinguishing source and target dynamics using limited target rollouts. Cross-domain alignment methods such as GAMA \citep{kim2020domain} and Cycle-Consistent IL \citep{raychaudhuri2021cross} jointly interact with source and target domains to align MDPs through latent or state–action correspondences. Another line of work performs few-shot domain adaptation, where a source-trained policy is fine-tuned on a handful of target demonstrations to bridge domain gaps \citep{yu2018one,brohan2022rt}. While effective, these approaches still rely on interaction or target-domain data, which is impractical in settings like autonomous driving and surgical robotics, where rollouts are unsafe or costly \citep{ILSurveyAutonomous}, or when no expert exists under new dynamics.

Motivated by the discussed limits, we study \emph{Distributionally Robust Offline Imitation Learning (DROIL)}. The goal of DROIL is to learn a \emph{robust} policy, using only expert demonstrations from a single nominal environment \emph{without any further interaction}, that remains effective under perturbations in transition dynamics. This enables IL agents to generalize across real-world dynamics variations while training entirely offline from one source of expert data, avoiding unsafe or costly rollouts. The setting is uniquely challenging: (i) standard IL methods such as Behavioral Cloning and Distribution Matching~\citep{firstBC,Kostrikov2020Imitation} implicitly assume identical train–test dynamics, so robustness must be induced purely through the learning objective, as no further interactions are available to infer or adapt to perturbed transitions; (ii) Robustification under transition uncertainty yields a minimax problem whose inner maximization is challenging to solve, as it depends on unknown perturbed dynamics. 

To address these challenges, we formulate robustness as a minimax optimization over an ambiguity set centered on the nominal transition dynamics and reformulate it in the $f$-divergence–based occupancy measure space. To ensure optimization only over stationary distributions, we impose a \emph{Balance Equation} derived from Bellman flow consistency. By introducing a \emph{triplet occupancy representation} and invoking strong duality, we transform the intractable adversarial maximization into a closed-form importance-weighting objective defined entirely under nominal expert data. 

Our main contributions include: (i) \emph{BE-DROIL}, the first distributionally robust imitation learning framework in the strict offline regime under Balance Equation constraint, using demonstrations from a single nominal environment;
(ii) a novel theoretically grounded triplet-occupancy representation that eliminates explicit dependence on unknown dynamics, enabling robustness purely in the occupancy space;
(iii) a scalable and practical alternating optimization algorithm based on converting the robust optimization into a tractable closed-form importance-weighting scheme under data collected from expert in nominal environment; and
(iv) strong empirical robustness across continuous-control benchmarks under diverse transition perturbations.

\section{Related Work}
\noindent\textbf{Offline Imitation Learning.}
Behavioral Cloning (BC) is a classical baseline in offline imitation learning \citep{firstBC}, formulating imitation as supervised learning from expert trajectories that map states to actions.
However, by disregarding environment dynamics and often operating with limited expert data, BC is prone to covariate shift and compounding error, which limit its generalization ability \citep{bcLimitation1}.
To incorporate transition structure, the \emph{\underline{Di}stribution \underline{C}orrection \underline{E}stimation (DICE)} family \citep{Kostrikov2020Imitation,kim2022demodice,mao2024odice} estimates stationary occupancy ratios without requiring explicit access to the behavior or learner policy, enabling off-policy evaluation and control under the stationary distribution. 

Our formulation builds on this occupancy-based view but differs fundamentally: its importance weights stem from a robust inner maximization over transition uncertainty rather than policy mismatch. Recent studies \citep{agrawal2024policy,agrawal2025markov,pmlr-v283-agrawal25a} show that offline IL must satisfy a Markov balance relation linking the expert policy and transition dynamics. Still, BC, DICE, and Markov-balance methods all assume identical train–test dynamics and therefore fail under transition shifts.

\noindent\textbf{Robust Reinforcement Learning.} This line of work tackles transition uncertainty by optimizing worst-case returns over an ambiguity set of transition kernels. Classical approaches use rectangular uncertainty sets and min–max dynamic programming \citep{nilim2005robust,iyengar2005robust}, while later works extend to distributionally robust formulations based on total variation, $f$-divergence, or Wasserstein metrics \citep{derman2020distributional,yu2023fast,pmlr-v283-panaganti25a}. These methods require reward feedback and often online rollouts to improve robustness. Offline variants such as robust fitted Q-iteration estimate worst-case value functions from fixed datasets using nominal-measure reformulations \citep{panaganti2022robust}. However, they still rely on known rewards, making them unsuitable for imitation learning, where only expert demonstrations are available. Our work builds on these robust foundations but reinterprets them for IL, without rewards or interaction, via a fully offline nominal-data formulation robust to transition uncertainty.

\noindent\textbf{Robust Imitation Learning.}
Much of IL research focuses on robustness to \emph{demonstration imperfections} (e.g., suboptimal or noisy labels) and stability in supervised imitation. Several methods learn from imperfect or negative demonstrations or denoise labels to mitigate covariate shift without altering dynamics \citep{pmlr-v97-wu19a,pmlr-v130-tangkaratt21a}, while others improve stability through noise injection or corrective querying \citep{pmlr-v78-laskey17a,pmlr-v15-ross11a}. Meta-IL and multi-task IL/IRL frameworks enhance adaptability across related tasks via shared experience but still assume consistent transition dynamics between training and deployment \citep{finn2017one,james2018task,Zhou2020Watch}.
To address \emph{dynamics mismatch}, domain randomization (DR) perturbs physical parameters such as mass or friction during training to improve sim-to-real transfer \citep{peng2018sim,huang2021generalization}, yet it requires a high-fidelity simulator, impractical in strictly offline IL where only nominal expert data exist.
Recent methods enhance BC robustness under different failure modes: \citet{wu2025robust} enforce Lipschitz regularization to improve stability against input perturbations, while DRIL–DICE \citep{seo2024mitigating} mitigates covariate shift by introducing $f$-divergence regularization to align state–action distributions under nominal dynamics. However, both assume fixed train–test transition models. Among works explicitly considering transition uncertainty, RIME learns policies that generalize across families of MDPs via online rollouts for robustness estimation \citep{chae2022robust}. The most closely related method to ours is \emph{Distributionally Robust Behavior Cloning (DRBC)}, which formulates behavior cloning under total-variation ambiguity sets around the transition model and optimizes for worst-case imitation loss \citep{panaganti2023distributionally}.
However, DRBC maximizes over arbitrary distributions in the uncertainty set, many of which do not correspond to stationary occupancy measures induced by the expert policy under any admissible transition kernel. In contrast, our framework constrains the adversary to operate only over \emph{valid stationary distributions} consistent with perturbed dynamics and generalizes beyond total variation to a broader class of $f$-divergence ambiguity sets, yielding a less pessimistic and more data-aligned robust objective.

\section{Preliminaries}
\noindent\textbf{The Imitation Learning Problem.} We consider an infinite-horizon discounted Markov decision process (MDP) $M = (\mathcal{S}, \mathcal{A}, T^{o}, r, \gamma, \mu)$,  
where $\mathcal{S}$ and $\mathcal{A}$ denote the state and action spaces,  
$T^{o}(s' \mid s, a)$ represents the transition dynamics,  
$r: \mathcal{S} \times \mathcal{A} \to \mathbb{R}$ is the reward function, $\gamma \in (0,1)$ is the discount factor, and $\mu$ is the initial state distribution such that $s_0 \sim \mu$. A policy $\pi$ defines a distribution over actions conditioned on the state,  
$\pi(a_t|s_t) = \Prob_{\pi}(A_t = a_t \mid S_t = s_t)$ for each time step $t$. The discounted occupancy measure induced by a policy $\pi$ on the $\mathcal{S}\times\mathcal{A}\times\mathcal{S}$ under the nominal transition kernel $T^{o}$ is defined as
$
d^{\pi}_{T^{o}}(s,a,s')
= (1-\gamma)\mathbb{E}_\pi\!\left[\sum_{t=0}^{\infty} \gamma^t\IndicatorFunc_{s_t = s, a_t = a, s_{t+1} = s'}\right],
$
where the expectation is taken over trajectories generated by  
$a_t \sim \pi(\cdot|s_t)$ and $s_{t+1} \sim T^{o}(\cdot|s_t, a_t)$,  
with $s_0$ drawn from $\mu$. The marginal state–action and state-only occupancy measures are then given by $d^{\pi}_{T^{o}}(s, a) = \sum_{s'} d^{\pi}_{T^{o}}(s,a,s')$ and $d^{\pi}_{T^{o}}(s) = \sum_{a,s'} d^{\pi}_{T^{o}}(s,a,s')$ respectively.

In the \emph{offline imitation learning (OIL)} setting,  
an agent has access only to demonstration trajectories collected from an expert policy $\pi_D$,  
denoted as $D = \{(s_0, a_0), (s_1, a_1), \ldots\}$,  
without any additional interaction with the environment.  
No reward information is provided in $D$,  
and the goal is to learn a \textit{policy} $\pi^\star_{oil}$ that mimics the expert’s behavior as closely as possible using this static dataset.  
Formally, the general learning objective is expressed as \citep{ImitationLearningPresentation}
\begin{equation}
    \pi^{\star}_{oil} \in \arg\min_{\pi \in \Pi}\;
    \mathbb{E}_{s \sim d^{\pi_D}_{T^{o}}}
    \big[\mathcal{L}(\pi(\cdot|s), \pi_D(\cdot|s))\big],
    \label{eq:ILObjective}
\end{equation}
where $\Pi$ denotes the set of stationary stochastic policies, and $\mathcal{L}$ is a divergence or distance metric (e.g., mean squared error or KL divergence) measuring the gap between learner and expert policies.

\vspace{0.5em}
\noindent\textbf{Distributionally Robust Setup.} We adopt the robust Markov decision process (RMDP) formulation~\citep{nilim2005robust}, defined by the tuple
$M_{\text{rob}} = (\mathcal{S}, \mathcal{A}, \mathcal{T}, r, \gamma, \mu)$,
which extends the standard MDP by introducing an uncertainty set $\mathcal{T}(\rho')$ over transition models.
The uncertainty set is factorized across all state–action pairs as
\begin{equation}
\begin{aligned}
\mathcal{T}(\rho') &= \mathop{\scalebox{0.9}{$\bigotimes$}}\limits_{(s,a) \in \mathcal{S} \times \mathcal{A}}
\mathcal{T}_{s,a}(\rho') \quad \quad \quad \\
\text{with} \quad
\mathcal{T}_{s,a}(\rho') &= \left\{ T_{s,a} \in \Delta(\mathcal{S}) \;:\; D_{TV}(T_{s,a}, T^{o}_{s,a}) \le \rho' \right\}.
\end{aligned}
\label{eqn:uncertainty_set}
\end{equation}
Here, $T^{o} = \left(T^{o}_{s,a}, {(s,a) \in \mathcal{S} \times \mathcal{A}}\right)$ denotes the nominal transition kernel,
$D_{TV}(\cdot, \cdot)$ measures the Total Variation distance between two distributions, $\Delta(\mathcal{S})$ is the probability distribution over $\mathcal{S}$,
and $\rho' > 0$ specifies the radius of the ambiguity set that determines the degree of robustness.
The \emph{distributionally robust offline imitation learning (DROIL)} problem is then defined as
\begin{equation}
\label{eqn:robust-il-tv}
\pi^{\star}_{droil} \in \arg\min_{\pi \in \Pi}\;
\max_{T \in \mathcal{T}}
\;
\E_{s \sim d^{\pi_D}_{T}} \big[\mathcal{L}(\pi(\cdot|s), \pi_D(\cdot|s))\big].
\end{equation}
where $d^{\pi_D}_T$ is the occupancy measure induced by expert policy $\pi_D$ in transition kernel $T$. The objective of this problem is to learn a \textit{robust policy} $\pi^{\star}_{droil}$ that minimizes the worst-case imitation loss evaluated over the state distributions induced by the expert policy $\pi_D$ under all admissible transition kernels $T \in \mathcal{T}(\rho')$.
The problem can equivalently be viewed as a game in which the adversary selects a transition kernel $T \in \mathcal{T}(\rho')$ that maximizes the imitation loss for a given learner’s policy $\pi$, while the learner seeks a policy $\pi^{\star}_{droil}$ that minimizes this adversarial objective.
In the offline setting, however, the key challenge is that expert demonstrations are collected only under the nominal transition kernel $T^{o}$, and no additional on-policy interactions are permitted in any $T \in \mathcal{T}(\rho')$.

\begin{definition}[f-divergence]
Let $f: \mathbb{R}_{+} \rightarrow \mathbb{R}$ be a continuous and convex function,  
and let $p, q \in \Delta(\mathcal{X})$ denote two probability distributions over a domain $\mathcal{X}$.  
The $f$-divergence between $p$ and $q$ is defined as
\begin{equation}
D_{f}(p \,\|\, q)
= \mathbb{E}_{x \sim q}\!\left[f\!\left(\frac{p(x)}{q(x)}\right)\right].
\end{equation}
A widely used example of an $f$-divergence is the Kullback–Leibler (KL) divergence,  
which arises when $f(u) = u \log u$.  
\end{definition}

\section{Methodology}
We begin by deriving a theoretical result that quantifies how uncertainty in the transition model influences the induced occupancy measure for a fixed policy.
\begin{lemma}
For any policy $\pi$ and transition kernel $T \in \mathcal{T}(\rho')$, the following inequality holds:
\[
D_{\mathrm{TV}}\!\left(d^{\pi}_{T}, \, d^{\pi}_{T^{o}}\right)
\le \frac{\rho'}{1 - \gamma},
\]
where $d^{\pi}_T, d^{\pi}_{T^0}$ are the $(s,a,s')$ occupancy measures induced by policy $\pi$ under transition kernels $T$ and $T^0$ respectively.
\label{lemma:D_TV_mismatch_s_a_s_prime}
\end{lemma}

Lemma~\ref{lemma:D_TV_mismatch_s_a_s_prime} shows that, for any admissible transition kernel $T \in \mathcal{T}(\rho')$, the occupancy measure over triplets $(s,a,s')$ induced by a policy $\pi$ deviates from that under the nominal dynamics $T^{o}$ by at most $\frac{\rho'}{1-\gamma}$ in Total Variation distance. The complete proof is provided in Appendix~\ref{appendix:theory}. Having established this bound, we can equivalently express the inner maximization in the \textit{DROIL} objective~\eqref{eqn:robust-il-tv} as a maximization over occupancy measures rather than transition kernels, which upper bounds the maximization over transition kernels, where the occupancy uncertainty set is defined as  
\begin{equation}
\begin{aligned}
\mathcal{D}^{\pi_D}(\rho) 
= 
\Big\{
d^{\pi_D}_{T} : 
D_{\mathrm{TV}}\!\left(d^{\pi_D}_{T},\, d^{\pi_D}_{T^{o}}\right)
\le \tfrac{\rho'}{1 - \gamma} = \rho
\Big\}.
\end{aligned}
\label{eqn:occupancy_uncertainty_set}
\end{equation}
As we will show next, optimizing directly over $\mathcal{D}^{\pi_D}$ in~\eqref{eqn:occupancy_uncertainty_set} enables us to formulate our constrained problem entirely in the space of occupancy measures, thereby eliminating explicit dependence on $T \in \mathcal{T}(\rho')$, which is unknown and cannot be sampled from in the purely offline setting.

\vspace{0.5em}
\noindent\textbf{Robust Policy Learning.} 
We now formulate the \textit{DROIL} problem entirely in the space of occupancy measures. Specifically, we consider the following constrained minimax optimization problem:
\begin{align}
    &\text{BE-DROIL} := 
    \min_{\pi} \; \max_{d^{\pi_D}_{T} \ge 0} 
    \; \E_{s \sim d^{\pi_D}_{T}}\!\left[\mathcal{L}(\pi(\cdot|s), \pi_D(\cdot|s))\right]
    \label{eq:rbc_dice_obj} \\
    \text{s.t.} \quad
    &\sum_{s'} d^{\pi_D}_{T}(s,a,s')
    = (1-\gamma)\mu(s)\pi_D(a|s)
    + \gamma \pi_D(a|s)\!
    \sum_{\tilde{s},\tilde{a}}
    d^{\pi_D}_{T}(\tilde{s},\tilde{a},s),
    \quad \forall (s,a) 
    \label{eqn:rbc_dice_constraint_a} \\
    & D_{f}\!\left(d^{\pi_D}_{T}(s,a,s') \,\|\, 
    d^{\pi_D}_{T^{o}}(s,a,s')\right)
    \le \rho.
    \label{eqn:rbc_dice_constraint_b}
\end{align}
Equation~\eqref{eq:rbc_dice_obj} corresponds to the \textit{DROIL} problem defined in \eqref{eqn:robust-il-tv}. The first constraint~\eqref{eqn:rbc_dice_constraint_a} enforces a \emph{Bellman flow conservation (\textbf{balance equation})} condition, ensuring that $d^{\pi_D}_{T}(s,a,s')$ represents a valid stationary occupancy measure realizable under the expert policy $\pi_D$ and the transition kernel $T$, thereby preserving the temporal dependencies that characterize the underlying MDP. For a detailed treatment of this consistency property, refer to \citep{puterman2014markov, altman2021constrained}. Enforcing this condition prevents the assignment of arbitrary probability mass to transitions that cannot occur under the expert’s policy in the underlying MDP. Based on the occupancy uncertainty set defined in~\eqref{eqn:occupancy_uncertainty_set}, the second constraint~\eqref{eqn:rbc_dice_constraint_b} bounds the deviation of the occupancy measure of the expert's policy $\pi_D$ under perturbed dynamics from that under nominal dynamics within an $f$-divergence ball of radius $\rho$. For $\rho=0$, the formulation reduces to non-robust imitation learning in \eqref{eq:ILObjective}, but with Bellman flow consistency enforced in the nominal environment.
Although we employ a general $f$-divergence formulation here, we later relate it to the Total Variation bound established in Lemma~\ref{lemma:D_TV_mismatch_s_a_s_prime}.

\vspace{0.5em}
\noindent\textbf{Why the Triplet $(s,a,s')$ Occupancy Measure.}
A key design choice in our formulation is to define the occupancy measure over triplets $(s,a,s')$, rather than the conventional state–action pair $(s,a)$.
If the occupancy were defined only on $(s,a)$, the Bellman flow constraint would explicitly depend on the transition kernel $T(s'|s,a)$ which is not known for any arbitrary $T \in \mathcal{T}(\rho')$ and cannot be estimated in offline settings where data is available only from a single nominal environment and no further interaction with the environment is permitted.
By incorporating the next state $s'$ directly into $d^{\pi_D}_{T}(s,a,s')$, the flow constraint can be expressed entirely in terms of occupancy measures, eliminating the need to deal with $T$ explicitly. As we will show next, this formulation therefore enables robust policy learning in the fully offline regime.

\vspace{0.5em}
\noindent\textbf{Lagrangian Formulation.} For a fixed learner policy $\pi$, we begin by considering the dual problem of  the inner maximization problem in \eqref{eq:rbc_dice_obj}-\eqref{eqn:rbc_dice_constraint_b}. This dual problem can be written as: 
\begin{equation}
\begin{aligned}
\min_{Q,\tau\ge 0}\; \max_{d^{\pi_D}_{T} \ge 0}
&\; \E_{s \sim d^{\pi_D}_{T}}\!\big[L_{\pi}(s)\big]
- \tau \!\left(
D_f\!\left(d^{\pi_D}_{T}\!\left(s,a,s'\right)\,\big\|\, d^{\pi_D}_{T^o}\!\left(s,a,s'\right)\right) - \rho
\right) \\
&\hspace{-0.5cm}-\sum_{s,a} Q(s,a)\!\left[
\sum_{s'} d^{\pi_D}_{T}(s,a,s') 
- (1-\gamma)\mu(s)\pi_D(a|s)
- \gamma \pi_D(a|s)\!\sum_{\tilde{s},\tilde{a}}
d^{\pi_D}_{T}(\tilde{s},\tilde{a},s)
\right].
\end{aligned}
\label{eqn:lagrangian_regularized_bc}
\end{equation}
where $L_{\pi}(s) = \mathcal{L}(\pi(\cdot|s), \pi_D(\cdot|s))$, $Q(s,a)\in\mathbb{R}$ denotes the Lagrange multiplier associated with the Bellman flow constraint, and $\tau\ge0$ is the multiplier for the $f$-divergence constraint.  

Because the inner maximization problem in~\eqref{eq:rbc_dice_obj}-\eqref{eqn:rbc_dice_constraint_b} is convex and admits a strictly feasible point (i.e. $d^{\pi_D}_{T^o}$), Slater’s condition~\citep{boyd2004convex} ensures that strong duality holds. 
Consequently, the dual problem in \eqref{eqn:lagrangian_regularized_bc} attains the same optimal value as the original maximization over $d^{\pi_D}_{T}$.

To simplify~\eqref{eqn:lagrangian_regularized_bc}, we expand and reorganize several terms as follows.
\begin{equation*}
\begin{aligned}
\sum_{s,a} Q(s,a)\!\left[(1-\gamma)\mu(s)\pi_D(a|s)\right]
&= (1-\gamma)\, \E_{s\sim\mu,\,a\sim\pi_D(\cdot|s)}[Q(s,a)], \\
\sum_{s,a} Q(s,a)\!\left[\gamma \pi_D(a|s)\sum_{\tilde{s},\tilde{a}} d^{\pi_D}_{T}(\tilde{s},\tilde{a},s)\right]
&= \sum_{s',a'} Q(s',a')\!\left[\gamma \pi_D(a'|s') \sum_{s,a} d^{\pi_D}_{T}(s,a,s')\right] \\
&= \sum_{s,a,s'} d^{\pi_D}_{T}(s,a,s')\,
   \Big[\gamma \sum_{a'} Q(s',a')\pi_D(a'|s') \Big] \\
&= \gamma\, \E_{s,a,s'\sim d^{\pi_D}_{T}}\!
   \left[\E_{a'\sim\pi_D(\cdot|s')}[Q(s',a')]\right], \\
\sum_{s,a} Q(s,a)\!\left[-\sum_{s'} d^{\pi_D}_{T}(s,a,s')\right]
&= -\, \E_{s,a,s'\sim d^{\pi_D}_{T}}[Q(s,a)].
\end{aligned}
\end{equation*}

Substituting these identities into~\eqref{eqn:lagrangian_regularized_bc} yields the simplified Lagrangian:
\begin{equation}
\begin{aligned}
\min_{Q,\tau\ge 0}\; \max_{d^{\pi_D}_{T} \ge 0}
&\; (1-\gamma)\E_{s\sim\mu,\,a\sim\pi_D(\cdot|s)}[Q(s,a)]
- \tau\, D_{f}\!\left(d^{\pi_D}_{T}(s,a,s') \,\|\, d^{\pi_D}_{T^{o}}(s,a,s')\right) + \rho\tau \\
&\quad + \E_{s,a,s'\sim d^{\pi_D}_{T}}\!\Big[L_{\pi}(s)
+ \gamma \E_{a'\sim\pi_D(\cdot|s')}[Q(s',a')] - Q(s,a)\Big].
\end{aligned}
\label{eqn:simplified_lagrangian_regularized_bc}
\end{equation}

This Lagrangian cannot be directly optimized using offline data because it depends on the unknown $d^{\pi_D}_{T}(s,a,s')$ for arbitrary $T \in \mathcal{T}(\rho')$.  
Since only expert trajectories under the nominal kernel $T^o$ are available, we define an importance ratio
\[
w(s,a,s') = \frac{d^{\pi_D}_{T}(s,a,s')}{d^{\pi_D}_{T^o}(s,a,s')}.
\]
Using the definition of $f$-divergence, we can rewrite~\eqref{eqn:simplified_lagrangian_regularized_bc} as
\begin{equation}
\begin{aligned}
\min_{Q,\tau\ge 0}\; \max_{w \ge 0}\; 
& (1-\gamma)\E_{s\sim\mu,\,a\sim\pi_D(\cdot|s)}[Q(s,a)] + \rho\tau \\
&\quad + \E_{s,a,s'\sim d^{\pi_D}_{T^o}}\!\left[-\tau\, f(w(s,a,s')) + w(s,a,s')\, e_{Q,\pi}(s,a,s')\right],
\end{aligned}
\label{eqn:final_lagrangian_constrained_bc}
\end{equation}
where $ 
e_{Q,\pi}(s,a,s') = L_{\pi}(s) + \gamma\E_{a'\sim\pi_D(\cdot|s')}[Q(s',a')] - Q(s,a).$

At this point, it is important to note that the original optimization problem in~\eqref{eq:rbc_dice_obj}–\eqref{eqn:rbc_dice_constraint_b} involved a maximization over unknown occupancy distributions $d^{\pi_D}_{T}$ for any $T \in \mathcal{T}(\rho')$.  
Through the above reformulation, we convert this intractable objective into one that depends only on samples drawn from the expert policy under the nominal dynamics $T^{o}$.  
This transformation is crucial, as it enables practical optimization in a strictly offline setting using expert demonstration data alone.
\begin{proposition}
\label{proposition:optimal_importance_weight}
For $\tau > 0$, the inner maximization in~\eqref{eqn:final_lagrangian_constrained_bc} admits the solution
\[
w^{\star}_{Q,\tau,\pi}(s,a,s') =
\max\!\left(0,\,(f')^{-1}\!\left(\frac{e_{Q,\pi}(s,a,s')}{\tau}\right)\right).
\]
For $\tau = 0$, $w^{\star}_{Q,\tau,\pi}(s,a,s') = +\infty$ if $e_{Q,\pi}(s,a,s')>0$ and $0$ otherwise.
\end{proposition}
The proof of Proposition~\ref{proposition:optimal_importance_weight} is provided in the Appendix \ref{appendix:theory}. Substituting $w^{\star}_{Q,\tau,\pi}$ from Proposition~\ref{proposition:optimal_importance_weight} into~\eqref{eqn:final_lagrangian_constrained_bc} simplifies the problem to a single minimization over $Q$ and $\tau$. This results in a two-stage learning procedure for the original optimization problem in~\eqref{eq:rbc_dice_obj}–\eqref{eqn:rbc_dice_constraint_b}, where the first stage estimates $(Q,\tau)$ under the nominal data distribution, and the second stage updates the policy parameters to minimize the expected weighted loss:
\begin{equation}
\begin{alignedat}{1}
\min_{Q,\tau\ge 0}\quad 
& (1-\gamma)\E_{s\sim\mu,\,a\sim\pi_D(\cdot|s)}[Q(s,a)] + \rho\tau \\
&\quad + \E_{(s,a,s')\sim d^{\pi_D}_{T^o}}\!\left[-\rho f\!\left(w^{\star}_{Q,\tau,\pi}(s,a,s')\right)
    + w^{\star}_{Q,\tau,\pi}(s,a,s')\, e_{Q,\pi}(s,a,s')\right] \\[3mm]
\min_{\pi}\quad 
& \E_{(s,a,s')\sim d^{\pi_D}_{T^o}}\!\left[w^{\star}_{Q,\tau,\pi}(s,a,s')\, L_{\pi}(s)\right].
\end{alignedat}
\label{label:final_optimization_problem}
\end{equation}

Optimization alternates between updating $(Q,\tau)$ and the policy $\pi$ until convergence, yielding a robust policy that minimizes expected loss under worst-case transition perturbations.

\noindent \textbf{Discussion on $f$-divergence.} 
From Lemma~\ref{lemma:D_TV_mismatch_s_a_s_prime}, we established a Total Variation (TV) bound between the occupancy measures induced by a policy under perturbed and nominal transition kernels. 
To extend this result beyond TV, we adopt an $f$-divergence formulation in our constrained optimization problem. 
To keep the constraint theoretically meaningful, Lemma~\ref{lemma:tv-dominated} ensures that any $f$-divergence whose generator is upper bounded by that of TV inherits the same upper bound between the occupancy measures induced by the same policy, as in Lemma~\ref{lemma:D_TV_mismatch_s_a_s_prime}. 
The proof is provided in Appendix~\ref{appendix:theory}. 
Proposition~\ref{proposition:optimal_importance_weight} further requires the generator $f(\cdot)$ to be differentiable with an invertible derivative $f'(\cdot)$ to compute $(f')^{-1}$. 
Since the TV generator $f_{\mathrm{TV}}(t)=\tfrac{1}{2}|t-1|$ is non-differentiable at $t=1$, we employ a smooth approximation, the \emph{SoftTV} generator~\citep{seo2024mitigating}, defined as
\begin{equation}
f_{\mathrm{SoftTV}}(x) = \tfrac{1}{2}\log\!\big(\cosh(x - 1)\big),
\qquad
\big(f'_{\mathrm{SoftTV}}\big)^{-1}(y) = \tanh^{-1}(2y) + 1.
\end{equation}
Lemma~\ref{lemma:softtv_leq_tv} in Appendix~\ref{appendix:theory} shows that $f_{\mathrm{SoftTV}}(x) \le f_{\mathrm{TV}}(x)$ for all $x$, and combining this with Lemma~\ref{lemma:tv-dominated} for $\alpha = 1$ gives
$
D_{f_{\mathrm{SoftTV}}}\!\left(d^{\pi}_{T}(s,a,s')\middle\|d^{\pi}_{T^{o}}(s,a,s')\right) \le \frac{\rho'}{1-\gamma}.
$
In our experiments, we use $f_{\mathrm{SoftTV}}(\cdot)$ as the generator for the $f$-divergence; however, any differentiable generator with an invertible derivative satisfying Lemma~\ref{lemma:tv-dominated} can be used.

\begin{lemma}
Let $f:[0,\infty)\!\to\!\mathbb{R}$ be an $f$-divergence generator with $f(1)=0$ and 
$f(t)\!\le\!\alpha\,f_{\mathrm{TV}}(t)$ for all $t\!\ge\!0$, where 
$f_{\mathrm{TV}}(t)=\tfrac{1}{2}|t-1|$ is the generator of total variation. 
Then, for any policy $\pi$ and transition kernel $T\!\in\!\mathcal{T}(\rho')$,
\[
D_f\!\left(d^{\pi}_{T}(s,a,s')\|d^{\pi}_{T^{o}}(s,a,s')\right)
\le \alpha\,D_{\mathrm{TV}}\!\left(d^{\pi}_{T},d^{\pi}_{T^{o}}\right)
\le \alpha\,\frac{\rho'}{1-\gamma}.
\]
\label{lemma:tv-dominated}
\end{lemma}

\begin{figure}[t]
\centering

\begin{tcolorbox}[
    width=\columnwidth,
    nobeforeafter, 
    coltitle=black, 
    fonttitle=\fontfamily{lmss}\selectfont\bfseries, 
    title=Hopper,
    halign title=flush center, 
    colback=backg_blue!5,
    colframe=brown!25,
    boxrule=1.5pt,
    left=0pt, 
    right=0pt,
    top=0pt,
    bottom=0pt,
    toptitle=0mm,
    bottomtitle=0mm
]
    \centering
    \includegraphics[width=0.33\columnwidth]{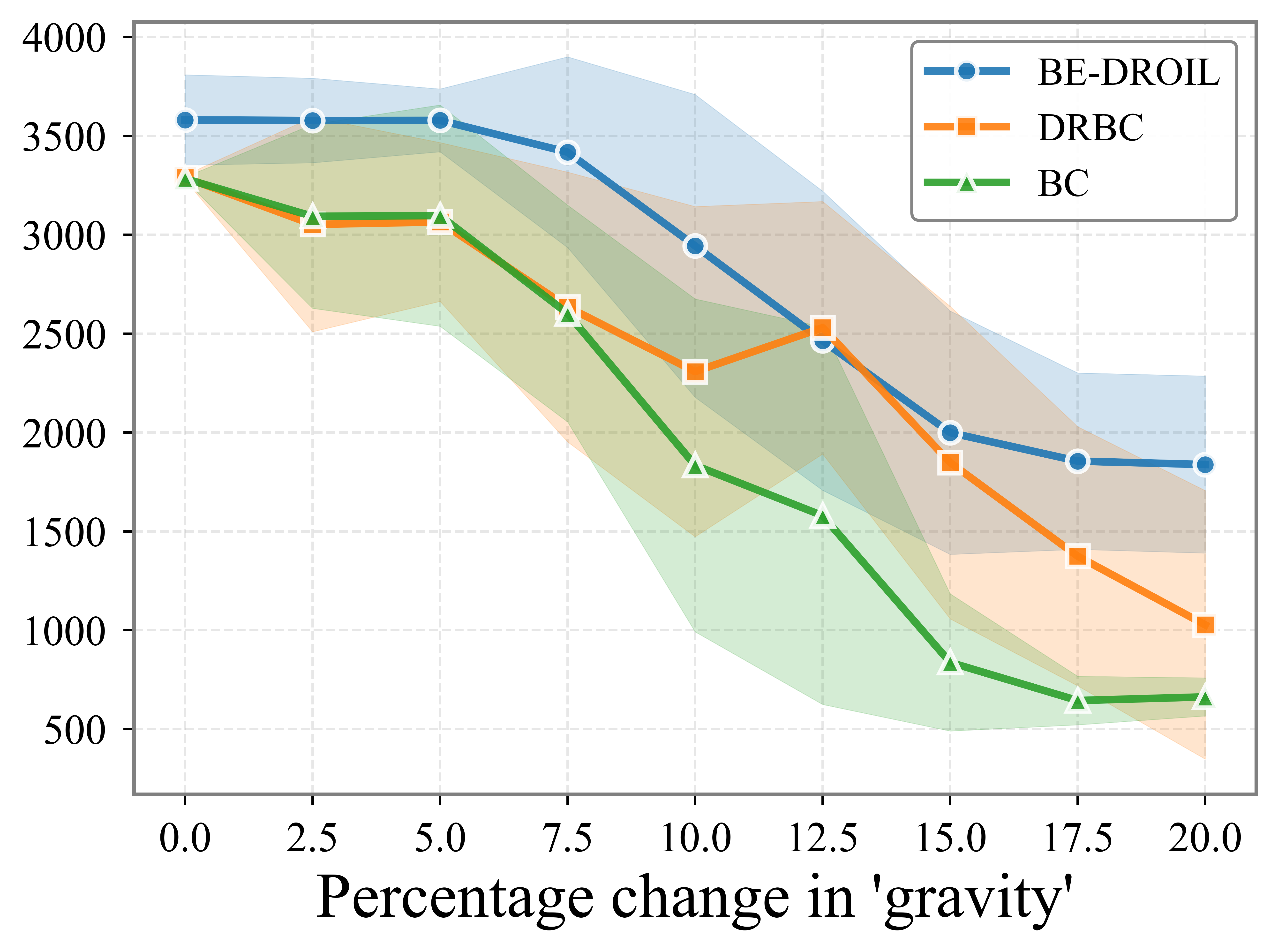}\hfill
    \includegraphics[width=0.33\columnwidth]{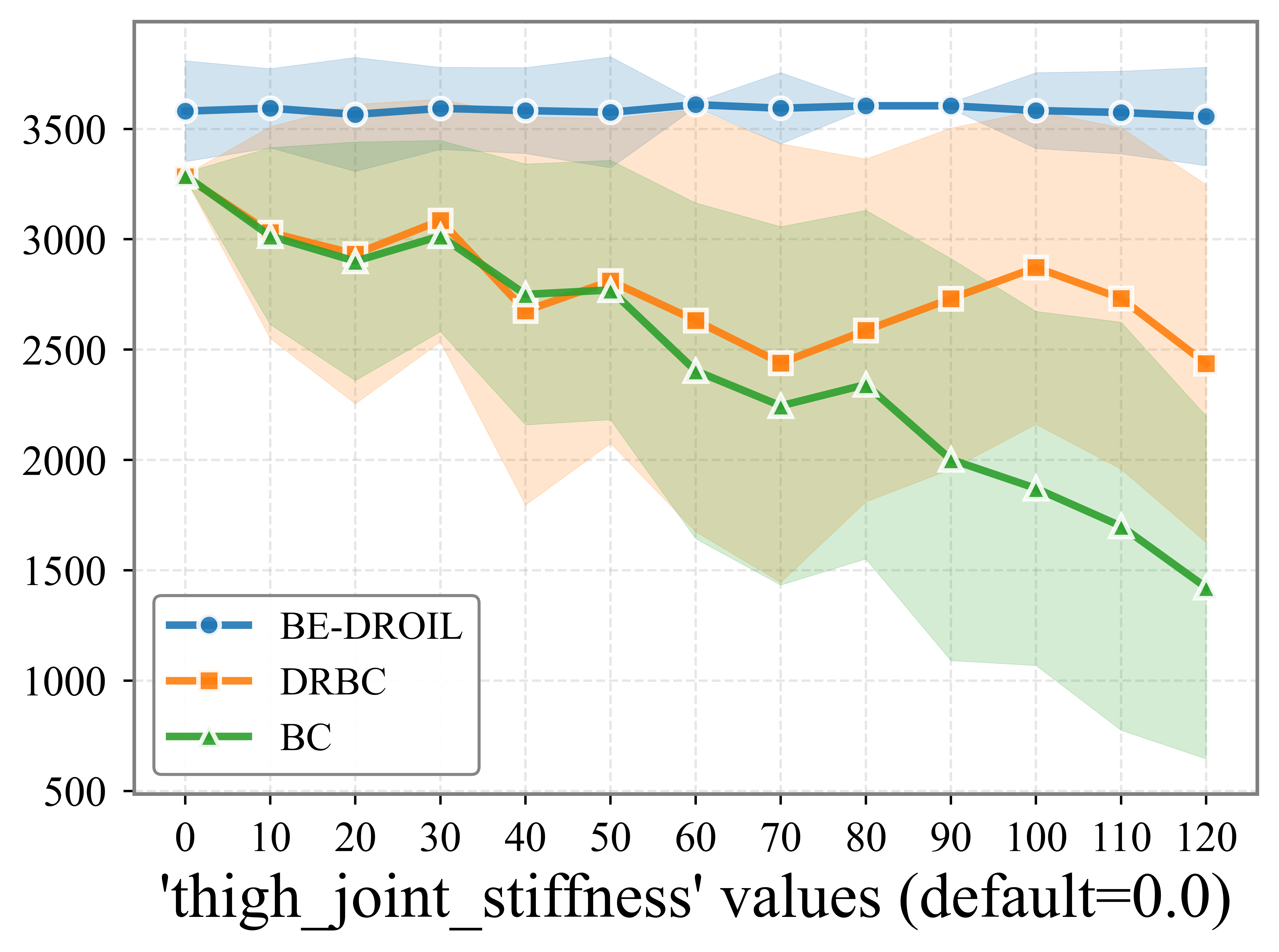}\hfill
    \includegraphics[width=0.33\columnwidth]{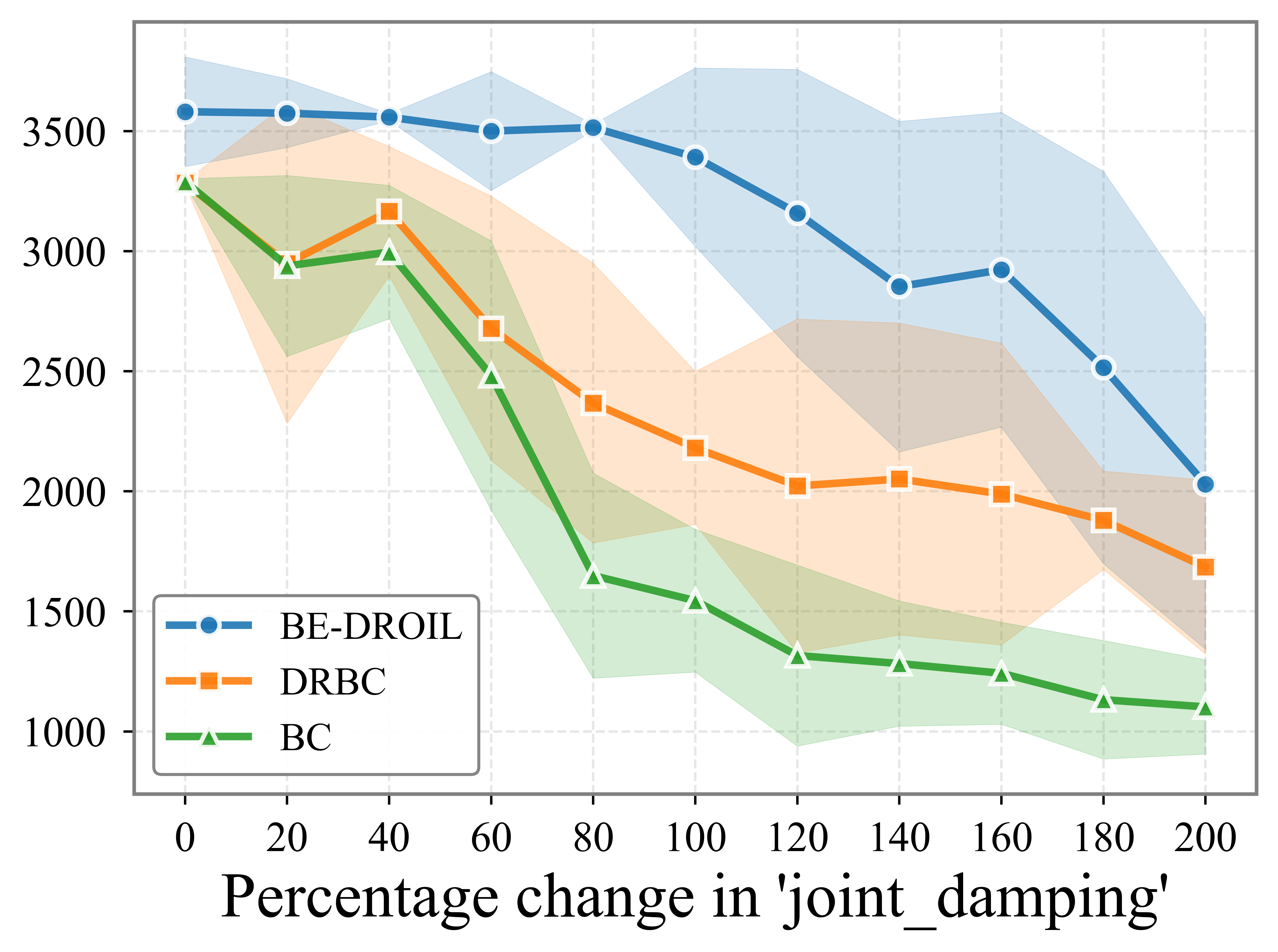}
\end{tcolorbox}

\vspace{0.1cm}  

\begin{tcolorbox}[
    width=\columnwidth,
    nobeforeafter, 
    coltitle=black, 
    fonttitle=\fontfamily{lmss}\selectfont\bfseries, 
    title=Ant,
    halign title=flush center, 
    colback=backg_blue!5,
    colframe=brown!25,
    boxrule=1.5pt,
    left=0pt, 
    right=0pt,
    top=0pt,
    bottom=0pt,
    toptitle=0mm,
    bottomtitle=0mm
]
    \centering
    \includegraphics[width=0.33\columnwidth]{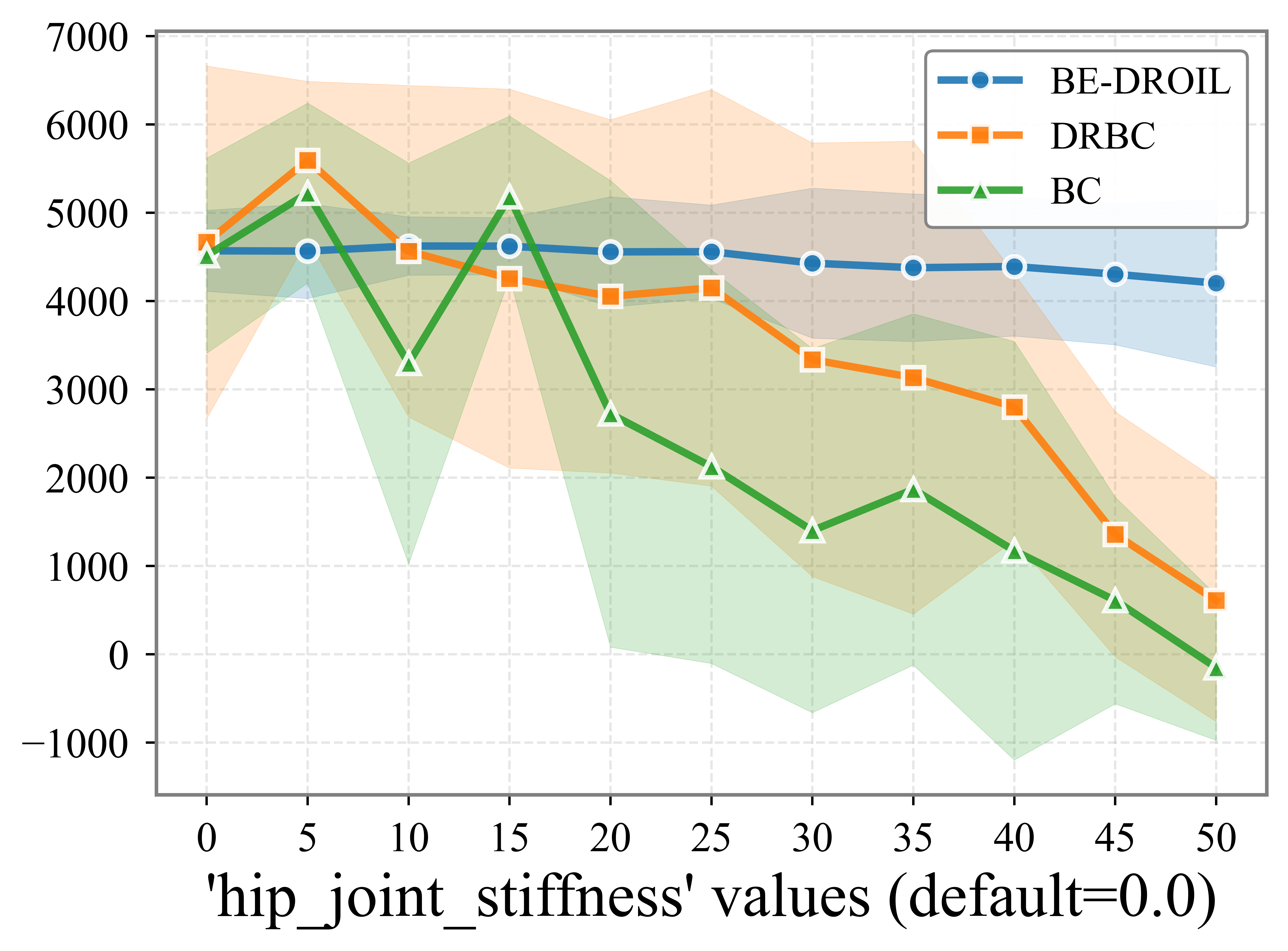}\hfill
    \includegraphics[width=0.33\columnwidth]{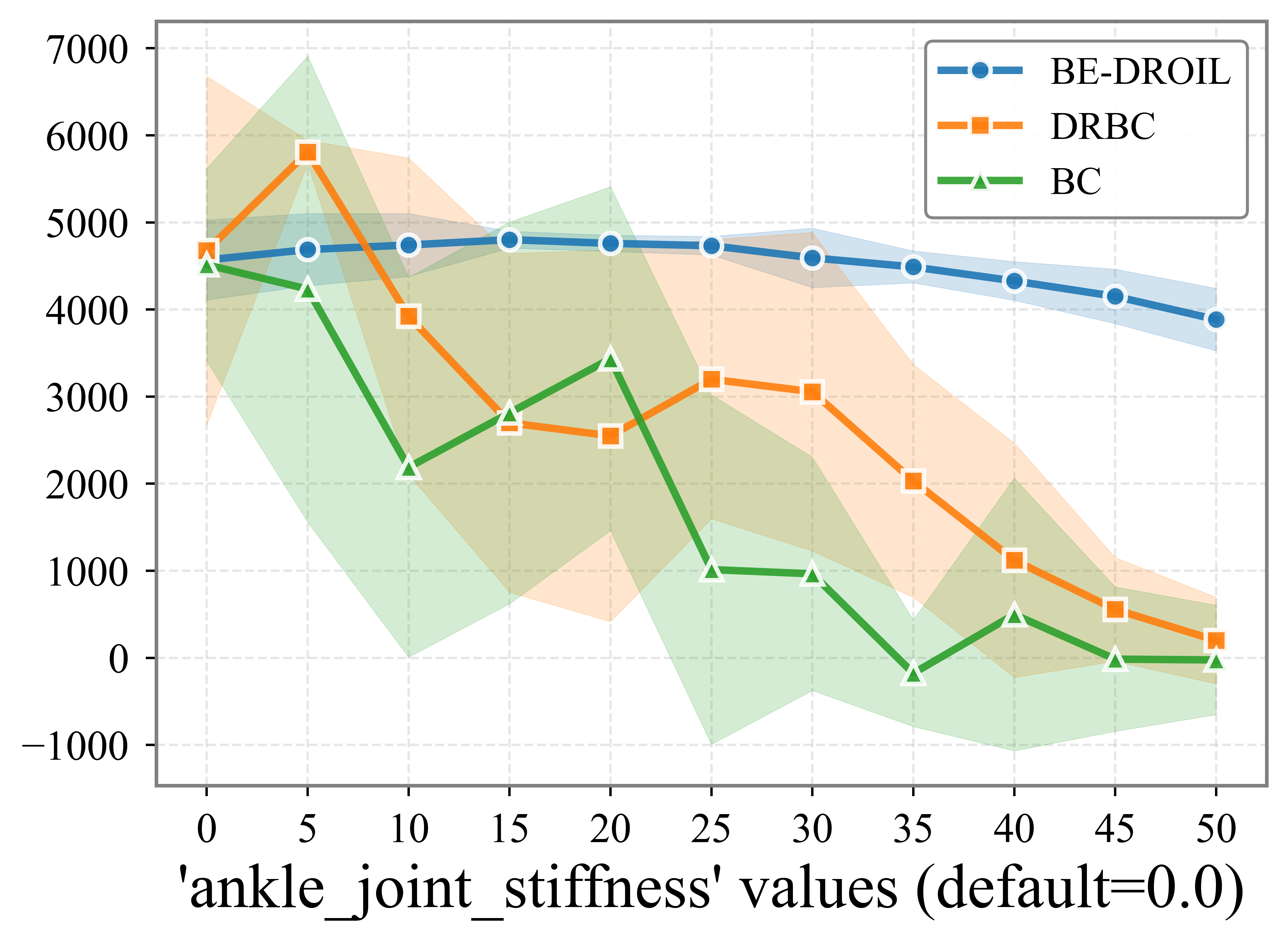}\hfill
    \includegraphics[width=0.33\columnwidth]{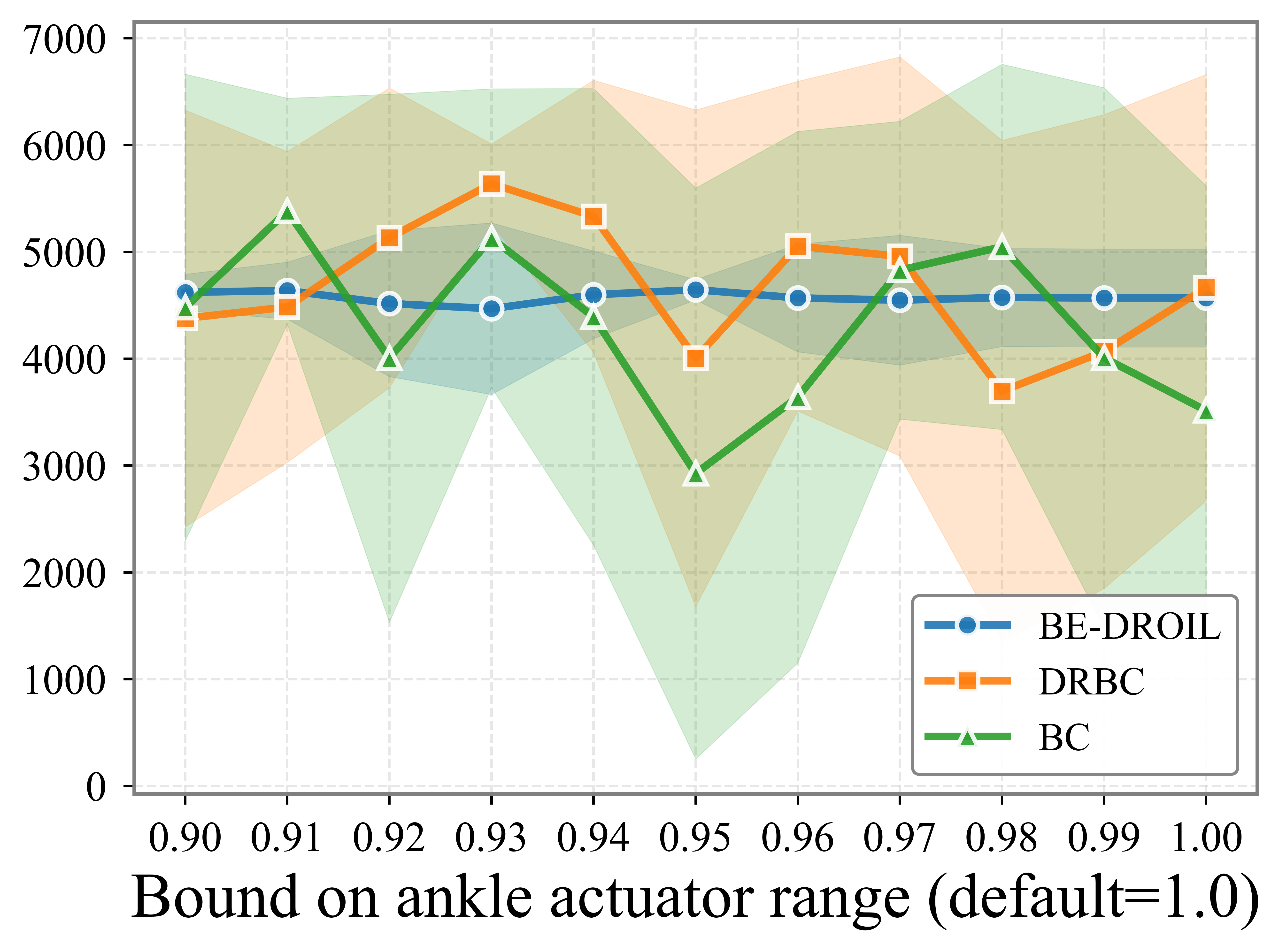}
\end{tcolorbox}
\vspace{-0.5cm}
\caption{Perturbation results for Hopper and Ant, with Y-axis denoting average cumulative reward.}
\label{fig:hopper_and_ant_perturbations}
\end{figure}

\section{Experimental Results}

\noindent\textbf{Setup.}
We evaluate the empirical performance of our proposed BE-DROIL algorithm on the MuJoCo locomotion suite \citep{todorov2012mujoco}, a standard benchmark for imitation learning in continuous-control domains with varying levels of difficulty. To construct the expert demonstration dataset, we follow the protocol of \citet{panaganti2023distributionally}, using expert trajectories generated by pre-trained TD3 \citep{fujimoto2018addressing} policies available in the RL Baselines3 Zoo repository \citep{rl-zoo3}. The number of samples per environment is also aligned with \citet{panaganti2023distributionally} to ensure fair comparison. Additional training details are provided in the Appendix \ref{appendix:experiments}.

\vspace{0.5em}
\noindent\textbf{Baseline Algorithms.}
We compare BE-DROIL against two baselines. Since fully offline robust imitation learning, where demonstrations originate from a single nominal environment, remains relatively unexplored, the only comparable prior method is Distributionally Robust Behavior Cloning (DRBC) \citep{panaganti2023distributionally}. In addition, we include standard Behavioral Cloning (BC) \citep{firstBC} as a non-robust baseline to highlight the importance of accounting for transition uncertainty when deployment dynamics differ from those seen during training.

\vspace{0.5em}
\noindent\textbf{Implementation.}
The term $(1-\gamma)\E_{s\sim\mu,,a\sim\pi_D(\cdot|s)}[Q(s,a)]$ in \eqref{label:final_optimization_problem} requires sampling from the initial-state distribution $\mu$.
Since expert datasets contain few unique initial states (e.g., only six in Walker2d), following standard DICE-based practice \citep{Kostrikov2020Imitation}, we treat every state within a trajectory as an effective initial state to ensure sufficient coverage for stable estimation.
Network architectures and remaining optimization details are provided in the Appendix \ref{appendix:experiments}.

\vspace{0.5em}
\noindent\textbf{Domains.} To evaluate robustness under transition shifts, we assess performance in perturbed test environments where key physical parameters are modified to induce model mismatch. Following the perturbation protocol of \citet{panaganti2023distributionally}, we vary parameters affecting torque generation, compliance, and energy dissipation, thereby capturing distinct modes of transition shift. Specifically, for Hopper, we perturb gravity, joint damping, and thigh-joint stiffness; for Ant, ankle- and hip-joint stiffness and actuator range; for Walker2d, gravity, actuator range, and foot-joint damping; and for HalfCheetah, back-joint stiffness, joint damping, and frictionloss.

\vspace{0.5em}
\noindent\textbf{Empirical Findings.}
Reported scores correspond to the mean and standard deviation of episodic returns over $100$ independently seeded rollouts.
Unlike DRBC, which tunes its hyperparameters (e.g., uncertainty radius and learning rate) separately for each environment, BE-DROIL uses a single set optimized on Ant and applies it unchanged across all others.
This design emphasizes generalization and fairness over environment-specific tuning.

\begin{figure}[t]
\centering
\begin{tcolorbox}[
    width=\columnwidth,
    nobeforeafter, 
    coltitle=black, 
    fonttitle=\fontfamily{lmss}\selectfont\bfseries, 
    title=Walker2d,
    halign title=flush center, 
    colback=backg_blue!5,
    colframe=brown!25,
    boxrule=1.5pt,
    left=0pt, 
    right=0pt,
    top=0pt,
    bottom=0pt,
    toptitle=0mm,
    bottomtitle=0mm
]
    \centering
    \includegraphics[width=0.33\columnwidth]{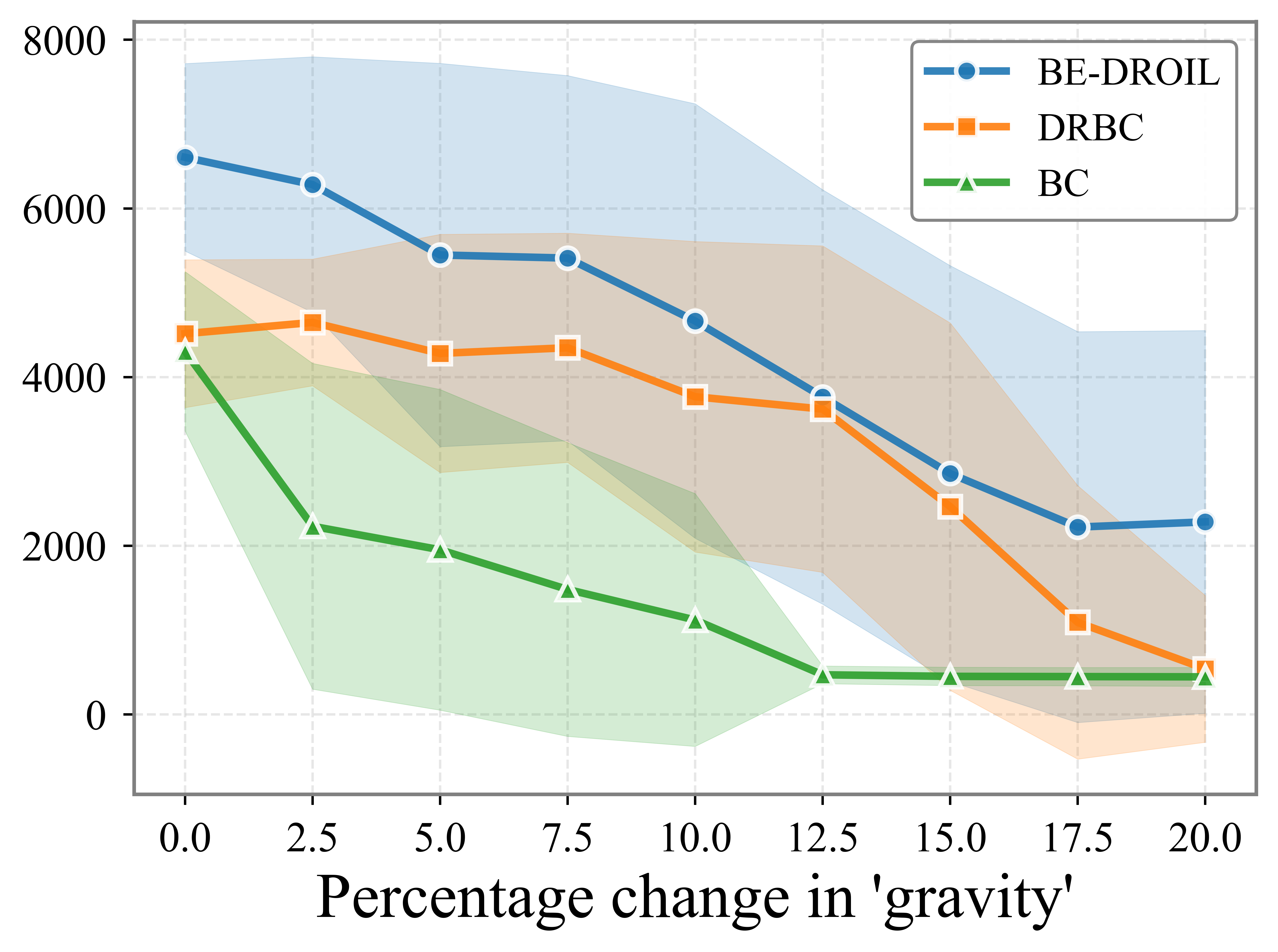}\hfill
    \includegraphics[width=0.33\columnwidth]{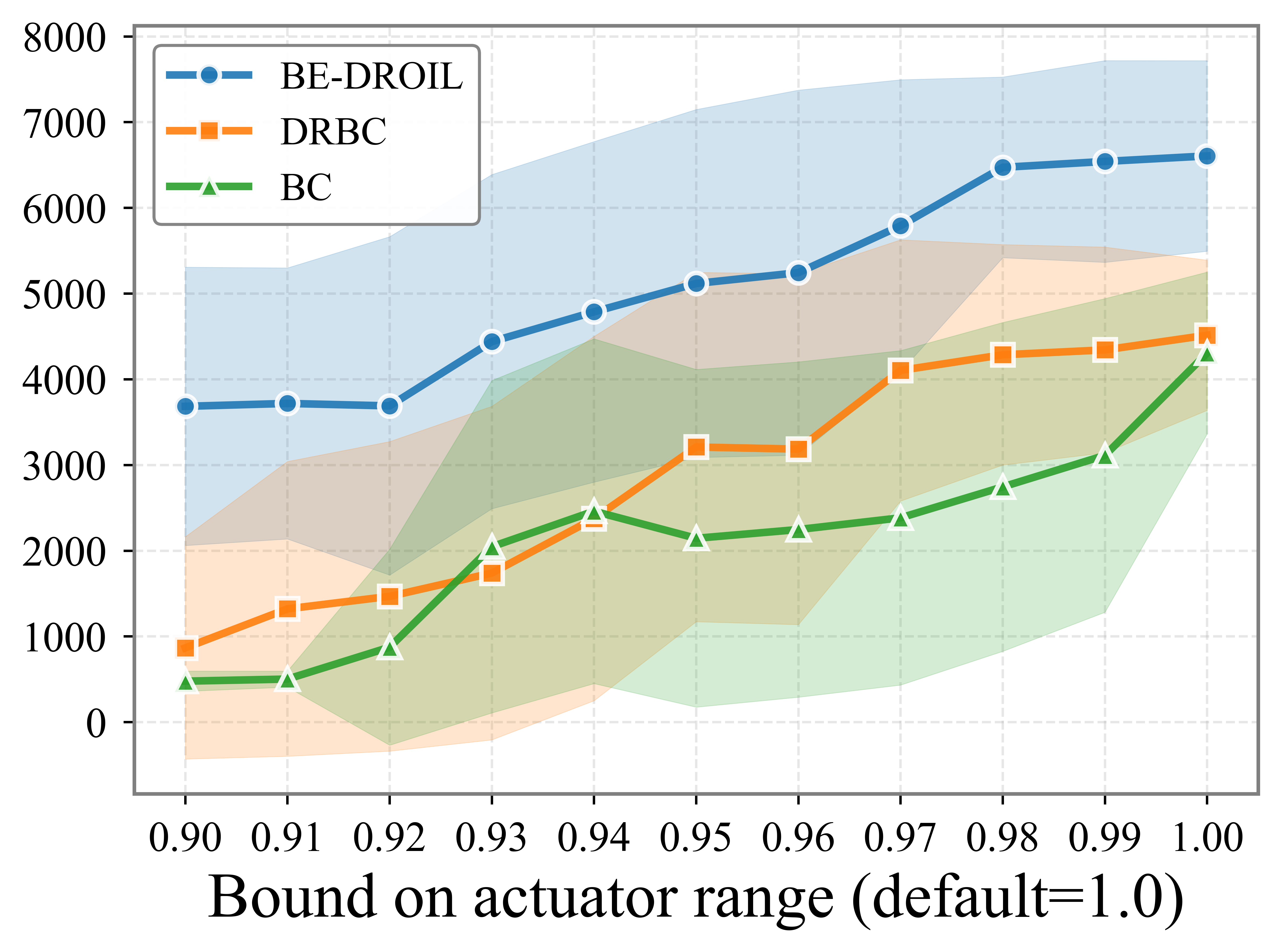}\hfill
    \includegraphics[width=0.33\columnwidth]{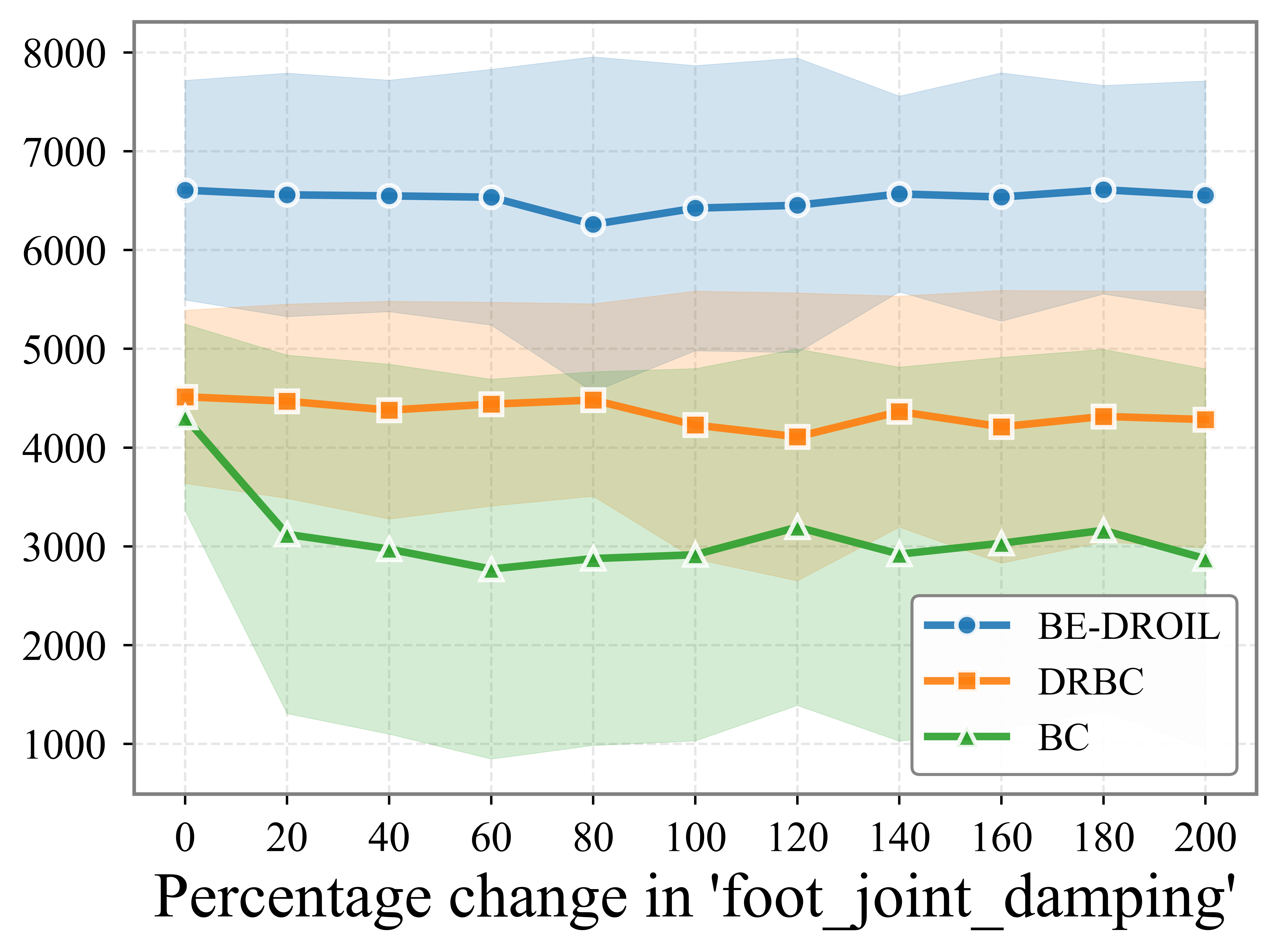}
\end{tcolorbox}

\vspace{0.1cm}  

\begin{tcolorbox}[
    width=\columnwidth,
    nobeforeafter, 
    coltitle=black, 
    fonttitle=\fontfamily{lmss}\selectfont\bfseries, 
    title=HalfCheetah,
    halign title=flush center, 
    colback=backg_blue!5,
    colframe=brown!25,
    boxrule=1.5pt,
    left=0pt, 
    right=0pt,
    top=0pt,
    bottom=0pt,
    toptitle=0mm,
    bottomtitle=0mm
]
    \centering
    \includegraphics[width=0.33\columnwidth]{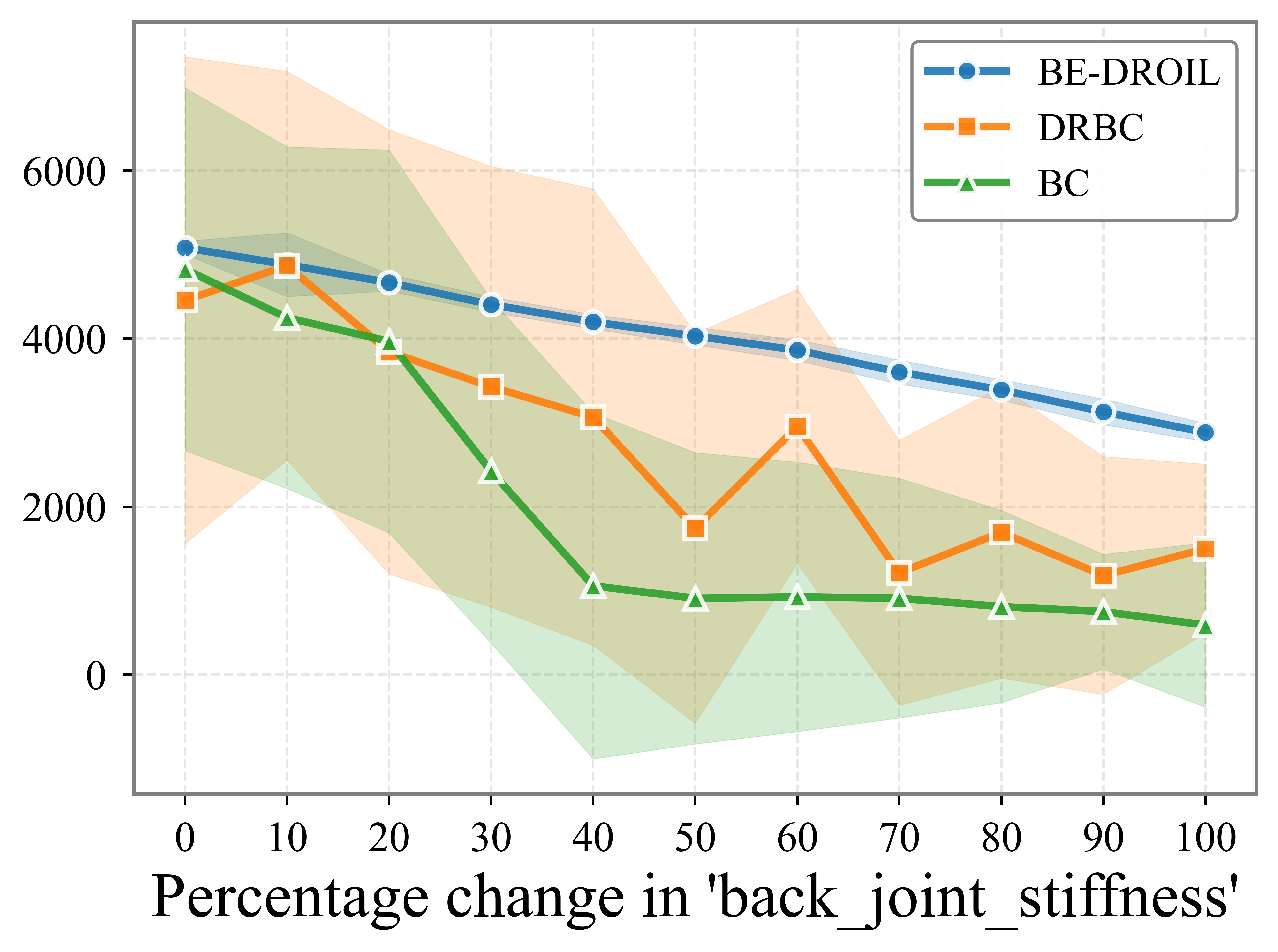}\hfill
    \includegraphics[width=0.33\columnwidth]{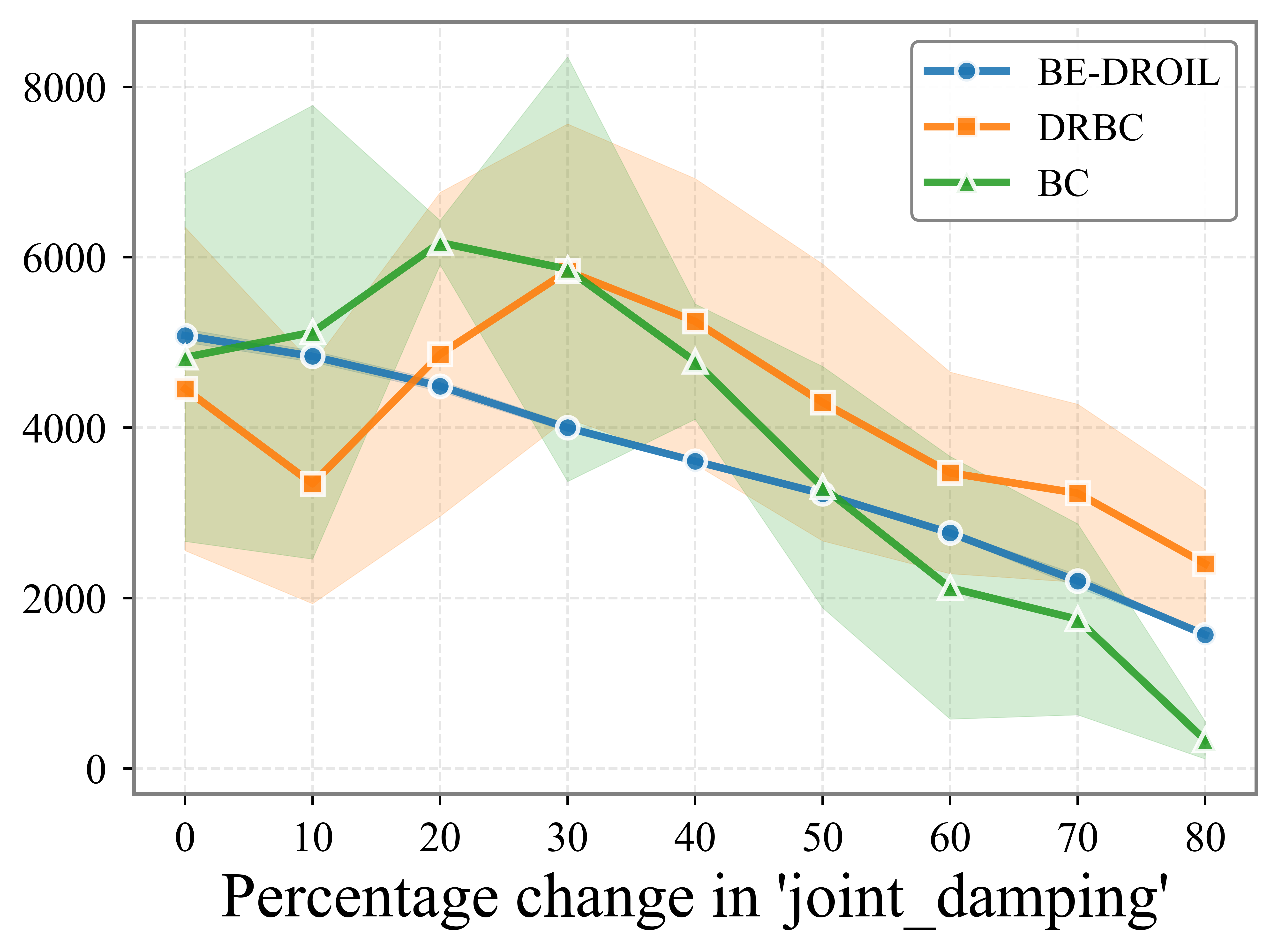}\hfill
    \includegraphics[width=0.33\columnwidth]{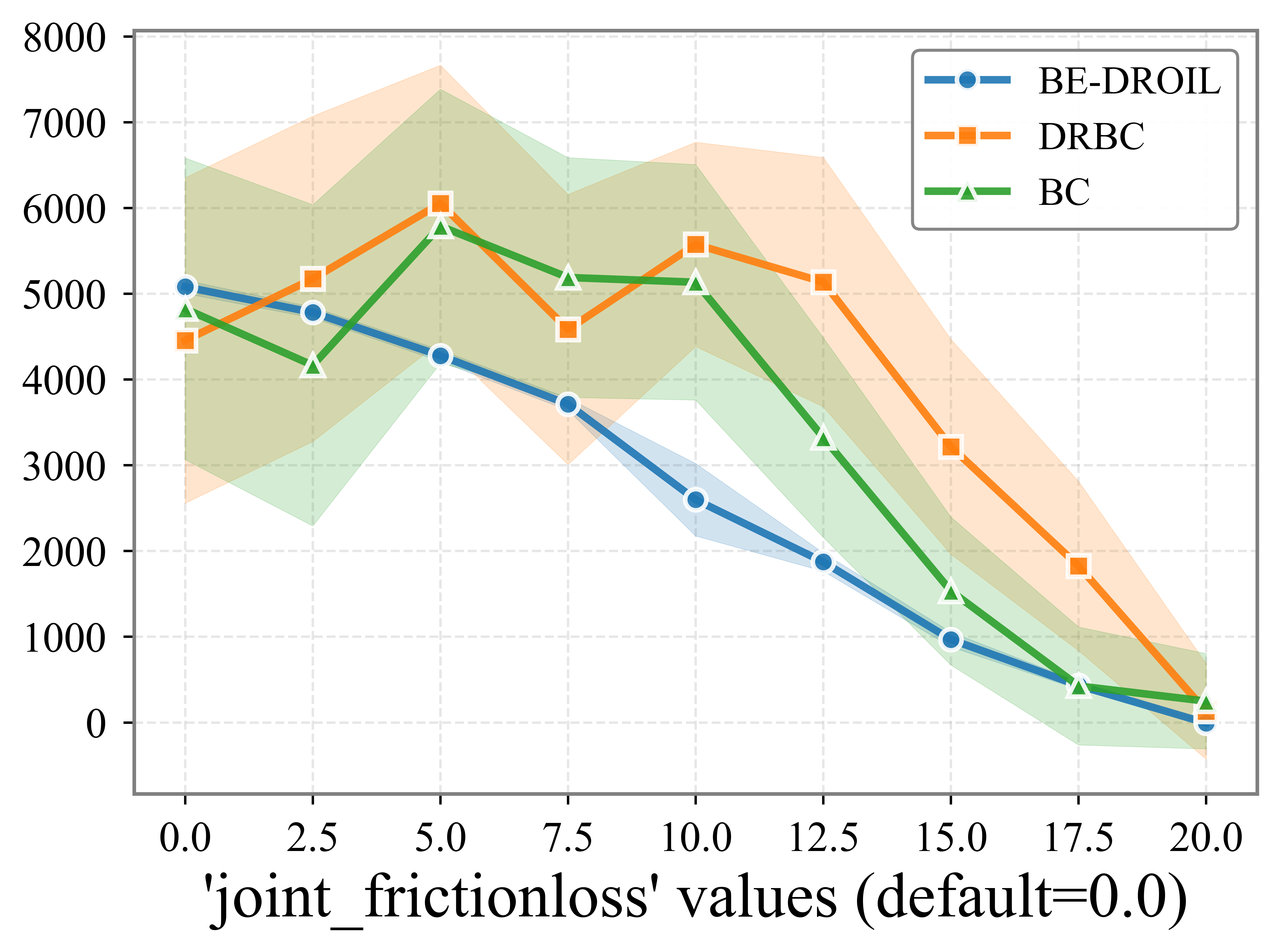}
\end{tcolorbox}
\vspace{-0.5cm}
\caption{Perturbation results for Walker2d and HalfCheetah, with Y-axis denoting average cumulative reward.}
\label{fig:walker_and_halfcheetah_perturbations}
\end{figure}

\vspace{0.5em}
\noindent\textbf{\textit{Hopper.}}
As gravity or damping increases, the agent encounters greater resistive forces and energy losses, making balance and propulsion harder to sustain. With higher stiffness, the leg becomes more rigid, reducing compliance and agility.
Figure~\ref{fig:hopper_and_ant_perturbations} shows that BC’s performance drops sharply, failing to remain upright even under moderate perturbations.
DRBC is more resilient but still declines steadily as shifts intensify. In contrast, BE-DROIL maintains high returns across all ranges, showing only gradual decay under extreme conditions while remaining consistently stable.

\vspace{0.5em}
\noindent\textbf{\textit{Ant.}}
When the actuator range is reduced, the limbs lose torque authority, making it harder for the agent to recover from posture disturbances involving multiple legs.
Figure~\ref{fig:hopper_and_ant_perturbations} shows that BC collapses, unable to maintain a coherent gait as torque and flexibility decrease.
DRBC declines moderately but retains basic stability across perturbations.
BE-DROIL, however, sustains nearly constant returns, preserving coordination and balance even when both actuation strength and joint stiffness are severely limited.

\vspace{0.5em}
\noindent\textbf{\textit{Walker2d.}}
Figure~\ref{fig:walker_and_halfcheetah_perturbations} shows that BC quickly loses stability as its legs fall out of sync, leading to frequent collapses.
DRBC preserves locomotion under moderate perturbations but gradually deteriorates as gravity and damping increase.
BE-DROIL remains the most stable across all conditions, maintaining coordinated walking and high returns with only mild decline even under strong environmental and control shifts.

\vspace{0.5em}
\noindent\textbf{\textit{HalfCheetah.}}
Figure~\ref{fig:walker_and_halfcheetah_perturbations} shows that BC’s performance drops sharply under these perturbations, while DRBC maintains greater stability.
BE-DROIL performs best under stiffness changes but exhibits steeper declines under increased damping and friction.
This gap likely stems from DRBC using environment-specific uncertainty radii and learning rates optimized for HalfCheetah, whereas BE-DROIL employs a fixed configuration across all domains.

\noindent Overall, BE-DROIL achieves state-of-the-art robustness on Ant, Hopper, and Walker2d, showing low variance and smaller degradation than baselines under actuator, stiffness, and gravity perturbations. Using a fixed hyperparameter configuration across all tasks ensures fair cross-domain evaluation but results in weaker performance under high damping and friction loss in HalfCheetah. Preliminary re-tuning indicates that minor adjustments to the learning rate and uncertainty radius largely close this gap, confirming BE-DROIL’s effectiveness when properly scaled. These results demonstrate principled distributional robustness under transition shifts, with strong generalization and consistent performance across domains without environment-specific tuning.

\section{Conclusion}

This work introduced BE-DROIL, a principled framework for \emph{distributionally robust imitation learning} under dynamics mismatch in the strictly offline setting, where only expert demonstrations from a single nominal environment are available. We formulated robust policy learning as a constrained minimax optimization over an $f$-divergence ambiguity set around the nominal transition kernel and derived an equivalent offline objective via a triplet occupancy formulation that eliminates explicit dependence on unknown dynamics. Using convex duality, we reduced the adversarial problem to a tractable importance-weighting scheme under nominal data, yielding an efficient alternating optimization algorithm. Empirically, BE-DROIL demonstrates consistent robustness gains over existing baselines across multiple continuous-control benchmarks with perturbed dynamics. Future work will extend this framework to non-stationary dynamics and large-scale imitation foundation models, as well as deriving analogues of Lemma~\ref{lemma:D_TV_mismatch_s_a_s_prime} for alternative divergence measures to broaden the class of admissible $f$-divergence generators in our optimization framework.

\section{Impact Statement}
This work introduces BE-DROIL, a framework for distributionally robust offline imitation learning that enhances policy reliability under dynamics shift. The method can reduce risky data collection in robotics or healthcare by learning entirely from offline demonstrations, but it may also propagate biases or unsafe behaviors present in expert data. Robustness guarantees are bounded by the chosen uncertainty set and may not hold under unmodeled or extreme shifts. The approach should therefore be deployed cautiously in real-world systems, with safety audits, drift monitoring, and human oversight. We will release code to ensure reproducibility while maintaining ethical and license compliance. Overall, BE-DROIL advances safer, more generalizable policy learning, provided users remain mindful of its assumptions and scope.

\medskip

\bibliographystyle{plainnat} 
\bibliography{references_neurips}

\newpage
\appendix

\section{Theoretical Derivations.}
\label{appendix:theory}

\begin{lemma}
For any policy $\pi$ and transition kernel $T \in \mathcal{T}(\rho')$, the following holds:
\begin{equation*}
    D_{\mathrm{TV}}(d^\pi_T(s), d^\pi_{T^o}(s)) \le \frac{\gamma \rho'}{1 - \gamma}.
\end{equation*}
\label{lemma:D_TV_mismatch_s}
\end{lemma}
\begin{proof}
See Lemma~7 in \citet{panaganti2023distributionally} for a detailed analysis.
\end{proof}

\begin{lemma}
For any policy $\pi$ and transition kernel $T \in \mathcal{T}(\rho')$, the following holds:
\begin{equation*}
    D_{\mathrm{TV}}(d^\pi_T(s,a), d^\pi_{T^o}(s,a)) \le \frac{\gamma \rho'}{1 - \gamma}.
\end{equation*}
\label{lemma:D_TV_mismatch_s_a}
\end{lemma}
\begin{proof}
We build upon the Lemma \ref{lemma:D_TV_mismatch_s} that quantifies how uncertainty in the transition model influences the induced occupancy measure on the state-space for a fixed policy. 

Since $d^\pi_T(s,a) = \pi(a|s) d^\pi_T(s)$, it follows that
\begin{align*}
D_{\mathrm{TV}}(d^\pi_T(s,a), d^\pi_{T^o}(s,a))
&= \frac{1}{2}\sum_{s,a} |d^\pi_T(s,a) - d^\pi_{T^o}(s,a)| \\
&= \frac{1}{2}\sum_{s,a} \pi(a|s) |d^\pi_T(s) - d^\pi_{T^o}(s)| \\
&= D_{\mathrm{TV}}(d^\pi_T(s), d^\pi_{T^o}(s)) \\
&\le \frac{\gamma \rho'}{1 - \gamma}.
\end{align*}
\end{proof}

Having computed the occupancy measure total-variation distance over state–action pairs for a fixed policy under two distinct transition kernels, we now restate Lemma~\ref{lemma:D_TV_mismatch_s_a_s_prime} and provide its proof.

\begin{customlem}{1}
Consider any policy $\pi$ and $T \in \mathcal{T}(\rho')$. Then,
\[
D_{\mathrm{TV}}\big(d^\pi_T(s,a,s'),\,d^\pi_{T^o}(s,a,s')\big) \le \frac{\rho'}{1 - \gamma}.
\]
\end{customlem}

\begin{proof}
We have
\[
d^\pi_T(s,a,s') = d^\pi_T(s,a) T_{s,a}(s'), \quad
d^\pi_{T^o}(s,a,s') = d^\pi_{T^o}(s,a) T^o_{s,a}(s').
\]
Then,
\begin{align*}
&\| d^\pi_T(s,a,s') - d^\pi_{T^o}(s,a,s') \|_1 \\
&= \sum_{s,a,s'} \left| d^\pi_T(s,a) T_{s,a}(s') - d^\pi_{T^o}(s,a) T^o_{s,a}(s') \right| \\
&\le \sum_{s,a,s'} d^\pi_T(s,a) \left| T_{s,a}(s') - T^o_{s,a}(s') \right|
   + \sum_{s,a,s'} \left| d^\pi_T(s,a) - d^\pi_{T^o}(s,a) \right| T^o_{s,a}(s').
\end{align*}

For the first term, using $\|T_{s,a} - T^o_{s,a}\|_1 = 2 D_{\mathrm{TV}}(T_{s,a}, T^o_{s,a}) \le 2\rho'$,
\[
\sum_{s,a,s'} d^\pi_T(s,a) \left| T_{s,a}(s') - T^o_{s,a}(s') \right|
= \sum_{s,a} d^\pi_T(s,a) \|T_{s,a} - T^o_{s,a}\|_1
\le 2\rho'.
\]

For the second term, since $\sum_{s'} T^o_{s,a}(s') = 1$, we have
\begin{equation*}
\begin{aligned}
\sum_{s,a,s'} \left| d^\pi_T(s,a) - d^\pi_{T^o}(s,a) \right| T^o_{s,a}(s')
&= \sum_{s,a} \left| d^\pi_T(s,a) - d^\pi_{T^o}(s,a) \right|\\
&= \| d^\pi_T(s,a) - d^\pi_{T^o}(s,a) \|_1.
\end{aligned}
\end{equation*}
By lemma \ref{lemma:D_TV_mismatch_s_a},
\[
D_{\mathrm{TV}}(d^\pi_T(s,a), d^\pi_{T^o}(s,a)) \le \frac{\gamma \rho'}{1 - \gamma},
\]
which implies
\[
\| d^\pi_T(s,a) - d^\pi_{T^o}(s,a) \|_1 \le \frac{2\gamma\rho'}{1 - \gamma}.
\]

Combining both terms,
\[
\| d^\pi_T(s,a,s') - d^\pi_{T^o}(s,a,s') \|_1
\le 2\rho' + \frac{2\gamma\rho'}{1 - \gamma}
= \frac{2\rho'}{1 - \gamma}.
\]

Hence,
\[
D_{\mathrm{TV}}\big(d^\pi_T(s,a,s'), d^\pi_{T^o}(s,a,s')\big)
= \frac{1}{2} \| d^\pi_T(s,a,s') - d^\pi_{T^o}(s,a,s') \|_1
\le \frac{\rho'}{1 - \gamma}.
\]
\end{proof}

\begin{customprop}{1}
When $\tau > 0$, for the inner maximization subproblem in \eqref{eqn:final_lagrangian_constrained_bc},
\begin{align*}
w^{\star}_{Q,\tau,\pi}(s,a,s') 
&:= \arg\max_{w \ge 0} 
\Bigl\{(1-\gamma)\,\E_{s\sim\mu,\,a\sim\pi_D(\cdot|s)}[Q(s,a)] + \rho\tau \\
&\quad + \E_{s,a,s'\sim d^{\pi_D}_{T^o}}\!\big[-\tau f(w(s,a,s')) + w(s,a,s')\,e_{Q,\pi}(s,a,s')\big]\Bigr\},
\end{align*}
the optimizer admits the closed-form expression
\[
w^{\star}_{Q,\tau,\pi}(s,a,s')
= \max\!\left(0,\,(f')^{-1}\!\left(\frac{e_{Q,\pi}(s,a,s')}{\tau}\right)\right),
\quad \forall (s,a,s'),
\]
where $(f')^{-1}$ denotes the inverse mapping of $f'$, which exists and is strictly increasing because $f$ is strictly convex. For $\tau = 0$, $w^{\star}_{Q,\tau,\pi}(s,a,s') = +\infty$ if $e_{Q,\pi}(s,a,s')>0$ and $0$ otherwise.
\end{customprop}
\begin{proof}
We begin with the case when $\tau > 0$. Define the objective functional
\begin{align*}
\mathcal{L}(Q,\tau,w)
&:= (1-\gamma)\E_{s\sim\mu,\,a\sim\pi_D(\cdot|s)}[Q(s,a)] + \rho\tau \\
&\quad + \E_{s,a,s'\sim d^{\pi_D}_{T^o}}\!\big[-\tau f(w(s,a,s')) + w(s,a,s')\,e_{Q,\pi}(s,a,s')\big].
\end{align*}
For fixed $(Q,\tau)$, the maximization $\max_{w\ge 0}\mathcal{L}(Q,\tau,w)$ can be viewed as a constrained optimization problem over $w$. 
Since the first two terms in $\mathcal{L}(Q,\tau,w)$ do not depend on $w$, the optimization effectively reduces to maximizing only the expectation term involving $w$. 
That is,
\begin{equation*}
\begin{aligned}
&\max_{w \geq 0} \quad \E_{s,a,s'\sim d^{\pi_D}_{T^o}}\!\big[-\tau f(w(s,a,s')) + w(s,a,s')\,e_{Q,\pi}(s,a,s')\big]
\\
\Leftrightarrow\;&\max_{w \geq 0} \quad \sum_{s,a,s'} d^{\pi_D}_{T^o}(s,a,s')\!\big[-\tau f(w(s,a,s')) + w(s,a,s')\,e_{Q,\pi}(s,a,s')\big].
\end{aligned}
\end{equation*}
Each term in the summation depends only on its corresponding local variable $w(s,a,s')$, 
which implies that the overall maximization can be solved independently for each $(s,a,s')$. 
Accordingly, for every $(s,a,s')$ with $d^{\pi_D}_{T^o}(s,a,s')>0$, we consider the scalar subproblem:
\begin{equation*}
\begin{aligned}
\max_{w(s,a,s')} \quad & -\,\tau f(w(s,a,s')) + e_{Q,\pi}(s,a,s')\,w(s,a,s'), \\
\text{s.t.}\quad & -\,w(s,a,s') \le 0.
\end{aligned}
\end{equation*}
Introducing the Lagrange multiplier $\kappa(s,a,s') \ge 0$ for the constraint $-w(s,a,s') \le 0$ yields the Lagrangian
\[
\mathcal{J}(w,\kappa)
= -\,\tau f(w(s,a,s')) + e_{Q,\pi}(s,a,s')\,w(s,a,s') - \kappa(s,a,s')\,(-w(s,a,s')).
\]
Simplifying,
\[
\mathcal{J}(w,\kappa)
= -\,\tau f(w(s,a,s')) + \big(e_{Q,\pi}(s,a,s') + \kappa(s,a,s')\big)\,w(s,a,s').
\]

\noindent
Because $f$ is strictly-convex and the feasible set $\{w\ge0\}$ is convex, the objective is concave in $w$, and thus strong duality holds. Consequently, the Karush--Kuhn--Tucker (KKT) conditions are both necessary and sufficient for optimality. We now state the \textbf{KKT} conditions corresponding to the above Lagrangian formulation.
\begin{itemize}[leftmargin=*]
    \item \textbf{Primal feasibility:} $w^*(s,a,s') \ge 0$.
    \item \textbf{Dual feasibility:} $\kappa^*(s,a,s') \ge 0$.
    \item \textbf{Complementary slackness:} $\kappa^*(s,a,s')\,w^*(s,a,s')=0$.
    \item \textbf{Stationarity:} The first-order condition is
    \begin{equation*}
        -\tau f'(w^*(s,a,s')) 
        + e_{Q,\pi}(s,a,s') + \kappa^*(s,a,s')=0.
    \end{equation*}
\end{itemize}

We now construct a candidate pair $(w^\star_{Q,\tau,\pi}, \kappa^\star)$ and show that it satisfies all KKT conditions:
\begin{equation}
\begin{aligned}
w^{\star}_{Q,\tau,\pi}(s,a,s')
&= \max\!\left(0,\,(f')^{-1}\!\left(\frac{e_{Q,\pi}(s,a,s')}{\tau}\right)\right), \\[4pt]
\kappa^{\star}(s,a,s') &=
\begin{cases}
0, & \text{if } (f')^{-1}\!\left(\frac{e_{Q,\pi}(s,a,s')}{\tau}\right) > 0, \\[6pt]
\tau f'(0) - e_{Q,\pi}(s,a,s'), & \text{otherwise.}
\end{cases}
\end{aligned}
\label{eqn:optimal_w_and_kappa}
\end{equation}

\paragraph{Verification of KKT conditions.}
By construction, $w^{\star}_{Q,\tau,\pi}(s,a,s') \ge 0$, ensuring primal feasibility. 
We now verify the remaining KKT conditions by considering two complementary cases.

\begin{itemize}
    \item \textbf{Case (1):} $(f')^{-1}\!\left(\tfrac{e_{Q,\pi}(s,a,s')}{\tau}\right) > 0$.  
    In this case, $w^{\star}_{Q,\tau,\pi}(s,a,s') = (f')^{-1}\!\left(\tfrac{e_{Q,\pi}(s,a,s')}{\tau}\right)$ and $\kappa^{\star}(s,a,s')=0$. It is straightforward to verify that $ w^{\star}_{Q,\tau,\pi}(s,a,s'), \kappa^{\star}(s,a,s')$  satisfy the KKT conditions in this case.

    \item \textbf{Case (2):} $(f')^{-1}\!\left(\tfrac{e_{Q,\pi}(s,a,s')}{\tau}\right) \leq 0$. In this case, $w^{\star}_{Q,\tau,\pi}(s,a,s') = 0$ and $\kappa^{\star}(s,a,s') = \tau f'(0) - e_{Q,\pi}(s,a,s')$. It is straightforward to verify that primal feasibility, complementary slackness and stationarity conditions are satisfied in this case. To verify dual feasibility, note that $(f')^{-1}\!\left(\tfrac{e_{Q,\pi}(s,a,s')}{\tau}\right) \leq 0$ implies that $\left(\tfrac{e_{Q,\pi}(s,a,s')}{\tau}\right) \leq f'(0)$ and hence $\kappa^{\star}(s,a,s') = \tau f'(0) - e_{Q,\pi}(s,a,s') \ge 0$. 
\end{itemize}

Therefore, the pair $(w^{\star}_{Q,\tau,\pi}, \kappa^{\star})$ defined in \eqref{eqn:optimal_w_and_kappa} satisfies all KKT conditions and is thus optimal. For those $(s,a,s')$ where $d^{\pi_D}_{T^o}(s,a,s') = 0$, the corresponding $w(s,a,s')$ values do not affect the objective and can therefore be chosen arbitrarily. For consistency, we define them in the same manner as for the tuples $(s,a,s')$ with $d^{\pi_D}_{T^o}(s,a,s') > 0$. Consequently, the closed-form optimal solution is
\[
w^{\star}_{Q,\tau,\pi}(s,a,s')
= \max\!\left(0,\,(f')^{-1}\!\left(\frac{e_{Q,\pi}(s,a,s')}{\tau}\right)\right),
\quad \forall (s,a,s').
\]

When $\tau = 0$, 
\begin{align*}
\mathcal{L}(Q,\tau,w)
&:= (1-\gamma)\E_{s\sim\mu,\,a\sim\pi_D(\cdot|s)}[Q(s,a)] + \E_{s,a,s'\sim d^{\pi_D}_{T^o}}\!\big[w(s,a,s')\,e_{Q,\pi}(s,a,s')\big].
\end{align*}

It is straightforward to see that maximizing this with respect to $w \geq 0$ gives

\[
w^{\star}_{Q,\tau,\pi}(s,a,s') =
\begin{cases}
+\infty, & \text{if } e_{Q,\pi}(s,a,s') > 0,\\[4pt]
0, & \text{if } e_{Q,\pi}(s,a,s') < 0,\\[4pt]
\text{arbitrary}, & \text{if } e_{Q,\pi}(s,a,s') = 0.
\end{cases}
\]

We consider the arbitrary value as $0$ to complete the proof.
\end{proof}

\begin{lemma}[Monotonicity of $f$-divergence]
Let $P$ and $Q$ be probability distributions on $\mathcal{X}$ with $P \ll Q$. If two generator functions $f$ and $g$ satisfy $f(x)\le g(x)$ for all $x\ge 0$, then
\[
D_f(P\|Q) \le D_g(P\|Q).
\]
\label{lemma:f_g_divergence_monotonicity}
\end{lemma}
\begin{proof}
Since $f(x)\le g(x)$ pointwise, it follows for $Q$-almost every $x$ that
$f\!\left(\frac{P(x)}{Q(x)}\right)\le g\!\left(\frac{P(x)}{Q(x)}\right)$.
Taking expectations with respect to $Q$ preserves the inequality:
\[
\mathbb{E}_{x\sim Q}\!\left[f\!\left(\tfrac{P(x)}{Q(x)}\right)\right]
\le
\mathbb{E}_{x\sim Q}\!\left[g\!\left(\tfrac{P(x)}{Q(x)}\right)\right].
\]
Hence $D_f(P\|Q)\le D_g(P\|Q)$, with equality only when the two functions agree 
$Q$-almost everywhere.
\end{proof}

\begin{customlem}{2}
Let $f:[0,\infty)\!\to\!\mathbb{R}$ be any $f$-divergence generator with $f(1)=0$. 
Assume there exists $\alpha\!\ge\!0$ such that
\[
f(t)\ \le\ \alpha\,f_{\mathrm{TV}}(t)\quad\text{for all }t\ge 0,
\qquad\text{where } f_{\mathrm{TV}}(t)=\tfrac12|t-1|.
\]
Then for any policy $\pi$ and any transition kernel $T\in\mathcal{T}$,
\[
D_f\!\left(d^{\pi}_{T}(s,a,s') \,\big\|\, d^{\pi}_{T^{o}}(s,a,s')\right)
\ \le\ \alpha\,D_{\mathrm{TV}}\!\left(d^{\pi}_{T}(s,a,s'),\, d^{\pi}_{T^{o}}(s,a,s')\right).
\]
In particular, combining with Lemma~\ref{lemma:D_TV_mismatch_s_a_s_prime} yields
\[
D_f\!\left(d^{\pi}_{T}(s,a,s') \,\big\|\, d^{\pi}_{T^{o}}(s,a,s')\right)
\ \le\ \alpha\,\frac{\rho'}{1-\gamma}.
\]
As a special case, if $f(t)\le f_{\mathrm{TV}}(t)$ for all $t\ge 0$ (i.e., $\alpha=1$), then
\[
D_f\!\left(d^{\pi}_{T}(s,a,s') \,\big\|\, d^{\pi}_{T^{o}}(s,a,s')\right)
\ \le\ \frac{\rho'}{1-\gamma}.
\]
\end{customlem}

\begin{proof}
The assumption $f(t)\le \alpha f_{\mathrm{TV}}(t)$ for all $t\ge0$ implies, by 
Lemma~\ref{lemma:f_g_divergence_monotonicity}, that
\[
D_f(P\|Q)\ \le\ \alpha\,D_{f_{\mathrm{TV}}}(P\|Q)
\ =\ \alpha\,D_{\mathrm{TV}}(P,Q).
\]
Applying this to 
$P=d^{\pi}_{T}(s,a,s')$ and $Q=d^{\pi}_{T^{o}}(s,a,s')$ gives
\[
D_f\!\left(d^{\pi}_{T}\|d^{\pi}_{T^{o}}\right)
\ \le\ \alpha\,D_{\mathrm{TV}}\!\left(d^{\pi}_{T},d^{\pi}_{T^{o}}\right).
\]
Finally, using Lemma~\ref{lemma:D_TV_mismatch_s_a_s_prime}, which bounds 
$D_{\mathrm{TV}}(d^{\pi}_{T}, d^{\pi}_{T^{o}})\le \rho'/(1-\gamma)$,
we obtain the desired inequality:
\[
D_f(d^{\pi}_{T}\| d^{\pi}_{T^{o}}) \le \alpha\,\frac{\rho'}{1-\gamma}.
\]
\end{proof}

\begin{lemma}
\label{lemma:softtv_leq_tv}
For all $x \in \mathbb{R}$, the soft total variation generator satisfies
\[
f_{\mathrm{SoftTV}}(x)
=\tfrac{1}{2}\log\!\big(\cosh(x-1)\big)
\le \tfrac{1}{2}|x-1|
=f_{\mathrm{TV}}(x).
\]
\end{lemma}

\begin{proof}
Let $t = x - 1$. Then $\cosh t = \tfrac{e^{t} + e^{-t}}{2}$. Using this, we have
\[
\cosh t
= \frac{e^{t} + e^{-t}}{2}
\le \frac{e^{|t|} + e^{|t|}}{2}
= e^{|t|}.
\]
Taking logarithms gives $\log(\cosh t) \le |t|$.  
Multiplying both sides by $\tfrac{1}{2}$ yields
\[
\tfrac{1}{2}\log(\cosh t) \le \tfrac{1}{2}|t|,
\]
and substituting back $t = x - 1$ completes the proof.
\end{proof}

\section{Experimental Settings}
\label{appendix:experiments}

\vspace{0.5em}
\noindent\textbf{Environments and Expert Data.}
We evaluate all algorithms on four continuous-control benchmarks from the MuJoCo suite \citep{todorov2012mujoco}: Hopper-v3, HalfCheetah-v3, Walker2d-v3, and Ant-v3, following the setup of \citet{panaganti2023distributionally}.  
Expert demonstrations are generated using pre-trained TD3 \citep{fujimoto2018addressing} policies available in the RL Baselines3 Zoo repository \citep{rl-zoo3}.  
The number of expert trajectories for each environment matches the DRBC configuration to ensure fair comparison and consistent dataset coverage across all methods.  
Details of the expert dataset sizes for each environment are provided in Table~\ref{tab:drbc_hyperparams}.

\vspace{0.5em}
\noindent\textbf{DRBC and BC Baselines.}
Both DRBC and BC are implemented in PyTorch using identical policy architectures composed of two hidden layers with 256 units and \texttt{tanh} activations.  
Each model is trained for $2\times10^6$ gradient steps with a batch size of 256, using the Adam optimizer. Learning rates and decay schedules follow \citet{panaganti2023distributionally}: the initial learning rate is $1\times10^{-4}$ for all environments except HalfCheetah ($1\times10^{-5}$), decayed exponentially by a factor of $0.9$–$0.95$ every $10,000$ steps.  
The DRBC robustness parameter $\rho_r$ is set to $\{0.2, 0.3, 0.5, 0.6\}$ for Hopper, HalfCheetah, Walker2d, and Ant respectively, as summarized in Table~\ref{tab:drbc_hyperparams}.  
BC uses the same architecture and optimizer settings but excludes the robust dual updates and uncertainty term.

\vspace{0.5em}
\noindent\textbf{BE-DROIL.}
Our proposed BE-DROIL algorithm is implemented in PyTorch. The policy is parameterized by a \texttt{TanhGaussianPolicy}, consisting of two fully connected layers with 256 hidden units
, followed by a Gaussian distribution transformed through $\tanh$ to enforce action bounds.  
Both the policy and $Q-$function are trained with a learning rate of $5\times10^{-5}$ using the Adam optimizer and a batch size of 512.  
The $Q-$function follows a two-layer MLP with 256 hidden units and \texttt{ReLU} activations. The discount factor is fixed at $\gamma = 0.99$. 

The loss function $L_{\pi}(\cdot)$ minimizes the mean-squared error between the learner’s and expert’s action distributions,
\[
L_{\pi}(s) = \mathrm{MSE}(a, a_D), \quad 
a \sim \pi(\cdot|s), \; a_D \sim \pi_D(\cdot|s),
\]
where $\pi$ and $\pi_D$ denote the learner and expert policies, respectively. Training proceeds for $1$ million gradient steps, alternating between one policy update and one joint update of the Q-function and $\tau$ network.  
Unlike DRBC, no environment-specific hyperparameter tuning is performed, i.e. identical configurations are used across all domains to ensure fair cross-environment evaluation.  
The complete BE-DROIL configuration is listed in Table~\ref{tab:troildice_hyperparams}.

\vspace{0.5em}
\noindent\textbf{Evaluation Protocol.}
Each trained policy is evaluated in perturbed test environments obtained by varying physical parameters such as gravity, actuator range, and joint stiffness to induce model mismatch.  
Performance is reported as the mean and standard deviation of episodic returns over $100$ rollouts with independent random seeds.  

\vspace{0.5em}
\begin{table}[t]
\centering
\resizebox{\linewidth}{!}{
\begin{tabular}{lcccccccccc}
\toprule
\textbf{Environment} & \textbf{Expert data size $N$} & \textbf{Robustness param. $\rho_r$} &
\textbf{Policy layers} & \textbf{Activation} &
\textbf{Max steps} & \textbf{Learning rate} &
\textbf{LR decay} & \textbf{Decay rate} & \textbf{Decay freq.} \\
\midrule
Hopper-v3      & 2000 & 0.2 & (256, 256) & Tanh & 2M & $1\times10^{-4}$ & True & 0.9  & 10k \\
HalfCheetah-v3 & 3000 & 0.3 & (256, 256) & Tanh & 2M & $1\times10^{-5}$ & True & 0.9  & 10k \\
Walker2d-v3    & 6000 & 0.5 & (256, 256) & Tanh & 2M & $1\times10^{-4}$ & True & 0.95 & 10k \\
Ant-v3         & 8000 & 0.6 & (256, 256) & Tanh & 2M & $1\times10^{-4}$ & True & 0.9  & 10k \\
\bottomrule
\end{tabular}
}
\caption{Hyperparameter configuration for DRBC across MuJoCo environments.}
\label{tab:drbc_hyperparams}
\end{table}

\vspace{0.5em}
\begin{table}[t]
\centering
\resizebox{\linewidth}{!}{
\begin{tabular}{lccccccccc}
\toprule
\textbf{Component} & \textbf{Architecture} & \textbf{Hidden units} & \textbf{Activation} & 
\textbf{Batch size} & \textbf{Learning rate} & \textbf{Steps} & \textbf{$\gamma$} &
\textbf{Update ratio (policy : Q, $\tau$)} \\
\midrule
Policy $\pi$ & TanhGaussian MLP & (256, 256) & ReLU & 512 & $5\times10^{-5}$ & 1M & 0.99 & 1:1 \\
Q-function $Q$ & MLP & (256, 256) & ReLU & 512 & $5\times10^{-5}$ & 1M & 0.99 & 1:1 \\
\bottomrule
\end{tabular}
}
\caption{Hyperparameter configuration for BE-DROIL.}
\label{tab:troildice_hyperparams}
\end{table}

\section{Details of \texorpdfstring{$f$}{f}-divergence Options}
\label{appendix:f-divergence}

We summarize several generator functions $f$ commonly employed in DICE-style objectives, along with their inverse gradients and the corresponding mappings from a scaled score $z := \frac{e_{Q,\pi}(s,a,s')}{\tau}$ to the optimal nonnegative weight $w^{\star}_{Q,\tau,\pi}(s,a,s')$. 
The latter is obtained from the first-order condition $f'(w) = z$ (projected to $w \ge 0$ when required). 
These definitions and mappings are provided in Table~\ref{tab:fdist-summary}.

We define the following helper functions:
\[
\mathrm{ReLU}(x) := \max\{0,x\}, 
\qquad
\mathrm{ELU}(x) := 
\begin{cases}
e^{x}-1, & x < 0,\\
x, & x \ge 0,
\end{cases}
\]
\[
f_{\text{soft-}\chi^2}(x) :=
\begin{cases}
x\log x - x + 1, & 0 < x < 1,\\[2pt]
(x-1)^2, & x \ge 1,
\end{cases}
\qquad
\bigl(f'_{\text{soft-}\chi^2}\bigr)^{-1}(y) :=
\begin{cases}
e^{\,y}, & y < 0,\\
y + 1, & y \ge 0.
\end{cases}
\]

\begin{table}[ht!]
\centering
\setlength{\tabcolsep}{8pt}
\renewcommand{\arraystretch}{1.25}
\begin{tabular}{lccc}
\toprule
\textbf{Divergence} & \boldmath$f(x)$ & \boldmath$\bigl(f'\bigr)^{-1}(y)$ & \boldmath$w^{\star}_{Q,\tau,\pi}$ for $z = \tfrac{e_{Q,\pi}}{\tau}$ \\
\midrule
Soft TV & $\tfrac12 \log\!\bigl(\cosh(x-1)\bigr)$
& $\tanh^{-1}(2y) + 1$
& $\mathrm{ReLU}\!\bigl(\tanh^{-1}(2z) + 1\bigr)$ \\

Total Variation & $\tfrac12 |x-1|$
& — 
& — \\

KL (forward) & $x\log x$
& $e^{\,y-1}$
& $\exp(z) - 1$ \\

$\chi^2$ & $\tfrac12 (x-1)^2$
& $y+1$
& $\mathrm{ReLU}(z + 1)$ \\

Soft $\chi^2$ & $f_{\text{soft-}\chi^2}(x)$
& $\bigl(f'_{\text{soft-}\chi^2}\bigr)^{-1}(y)$
& $\mathrm{ELU}(z) + 1$ \\
\bottomrule
\end{tabular}
\caption{Summary of common $f$-generators, their inverse gradients, and the closed-form mappings from the normalized score $z$ to the optimal weight $w^{\star}_{Q,\tau,\pi}$. A dash indicates that no simple closed form exists for that column.}
\label{tab:fdist-summary}
\end{table}

\end{document}